\documentclass[pmlr]{jmlr} % This is the PMLR single-column format

\usepackage[T1]{fontenc}
\usepackage[utf8]{inputenc}
\usepackage{amsmath}
\usepackage{amssymb}
\usepackage{bm}
\usepackage{booktabs}
\usepackage{algorithm2e}
\usepackage{url}
\usepackage{graphicx}
\graphicspath{{figures/}}
\usepackage{pgfplots}
\pgfplotsset{compat=1.18}
\usepgfplotslibrary{fillbetween}
\usetikzlibrary{backgrounds}
\usetikzlibrary{fit}
\usepackage[capitalise,nameinlink]{cleveref}

\usepackage[disable]{todonotes}

\pgfplotsset{compat=1.17}
\usetikzlibrary{intersections}
\usetikzlibrary{positioning}
\usepgfplotslibrary{fillbetween}
\usepgfplotslibrary{groupplots}
\usetikzlibrary{shapes.geometric,backgrounds}

\pgfdeclarelayer{background}    % declare background layer
\pgfsetlayers{background,main}  % set the order of the layers (main is the standard layer)
\definecolor{beige}{RGB}{245, 245, 220}

\definecolor{darkgrey}{RGB}{75, 75, 75}
\definecolor{lightgrey}{RGB}{250, 250, 250}

\usetikzlibrary{calc, arrows, fit, positioning, patterns, 
decorations.pathreplacing, shapes}
\tikzstyle{dash} = [dashed, -latex,>=latex]
\tikzstyle{line} = [draw, -latex,>=latex]
\tikzstyle{smallbox} = [draw, minimum size=5.0mm]
\tikzstyle{box} = [draw, minimum size=7.0mm]
\tikzstyle{bigbox} = [draw, minimum size=10.0mm]
\tikzstyle{blackbox} = [draw, minimum size=2.0mm, fill=black]
\tikzstyle{switch} = [trapezium, trapezium angle=120, draw, rotate=90,  
inner ysep=5pt, outer sep=5pt,
minimum height=7mm, minimum width=7mm]
\tikzstyle{roundbox} = [draw, circle, inner sep=0pt, minimum size=3mm]
\tikzstyle{clamped} = [draw, fill=black, minimum size=0.15cm]
\tikzstyle{msgcircle} = [shape=circle, draw, inner sep=0pt, minimum 
size=4mm, fill=white, font=\scriptsize]
\tikzstyle{darkmsgcircle} = [shape=circle, draw, inner sep=0pt, minimum 
size=4mm, fill=darkgrey, text=white, font=\scriptsize]
\tikzstyle{msgdoublecircle} = [shape=circle, double, double 
distance=1.5pt, draw, inner sep=0pt, minimum size=5mm, fill=white]
\tikzstyle{darkmsgdoublecircle} = [shape=circle, double, double 
distance=1.5pt, draw, inner sep=0pt, minimum size=5mm, fill=darkgrey, 
text=white, font=\bfseries]
\newcommand{\msg}[6]{
      \ifthenelse{\isin{#1}{left} \AND \isin{#2}{down}}{
            \coordinate (anchor) at ($({#3})!{#5}!({#4})$);
            \node[xshift=-6.0mm] at (anchor) {#6};
            \node[xshift=-1.0mm] at (anchor) {$\downarrow$};
      }{}
      \ifthenelse{\isin{#1}{right} \AND \isin{#2}{down}}{
            \coordinate (anchor) at ($({#3})!{#5}!({#4})$);
            \node[xshift=6.0mm] at (anchor) {#6};
            \node[xshift=1.0mm] at (anchor) {$\downarrow$};
      }{}

      \ifthenelse{\isin{#1}{down} \AND \isin{#2}{right}}{
            \coordinate (anchor) at ($({#3})!{#5}!({#4})$);
            \node[ yshift=-4.0mm] at (anchor) {#6};
            \node[yshift=-1.0mm] at (anchor) {$\rightarrow$};
      }{}
      \ifthenelse{\isin{#1}{up} \AND \isin{#2}{right}}{
            \coordinate (anchor) at ($({#3})!{#5}!({#4})$);
            \node[ yshift=4.0mm] at (anchor) {#6};
            \node[yshift=1.0mm] at (anchor) {$\rightarrow$};
      }{}

      \ifthenelse{\isin{#1}{down} \AND \isin{#2}{left}}{
            \coordinate (anchor) at ($({#3})!{#5}!({#4})$);
            \node[ yshift=-4.0mm] at (anchor) {#6};
            \node[yshift=-1.0mm] at (anchor) {$\leftarrow$};
      }{}
      \ifthenelse{\isin{#1}{up} \AND \isin{#2}{left}}{
            \coordinate (anchor) at ($({#3})!{#5}!({#4})$);
            \node[ yshift=4.0mm] at (anchor) {#6};
            \node[yshift=1.0mm] at (anchor) {$\leftarrow$};
      }{}

      \ifthenelse{\isin{#1}{left} \AND \isin{#2}{up}}{
            \coordinate (anchor) at ($({#3})!{#5}!({#4})$);
            \node[ xshift=-6.0mm] at (anchor) {#6};
            \node[xshift=-1.0mm] at (anchor) {$\uparrow$};
      }{}
      \ifthenelse{\isin{#1}{right} \AND \isin{#2}{up}}{
            \coordinate (anchor) at ($({#3})!{#5}!({#4})$);
            \node[ xshift=6.0mm] at (anchor) {#6};
            \node[xshift=1.0mm] at (anchor) {$\uparrow$};
      }{}
}

\newcommand{\msgcircle}[6]{
      \ifthenelse{\isin{#1}{left} \AND \isin{#2}{down}}{
            \coordinate (anchor) at ($({#3})!{#5}!({#4})$);
            \node[msgcircle,xshift=-5.0mm] at (anchor) {#6};
            \node[xshift=-1.5mm] at (anchor) {$\downarrow$};
      }{}
      \ifthenelse{\isin{#1}{right} \AND \isin{#2}{down}}{
            \coordinate (anchor) at ($({#3})!{#5}!({#4})$);
            \node[msgcircle,xshift=5.0mm] at (anchor) {#6};
            \node[xshift=1.5mm] at (anchor) {$\downarrow$};
      }{}

      \ifthenelse{\isin{#1}{down} \AND \isin{#2}{right}}{
            \coordinate (anchor) at ($({#3})!{#5}!({#4})$);
            \node[msgcircle, yshift=-5.0mm] at (anchor) {#6};
            \node[yshift=-2.0mm] at (anchor) {$\rightarrow$};
      }{}
      \ifthenelse{\isin{#1}{up} \AND \isin{#2}{right}}{
            \coordinate (anchor) at ($({#3})!{#5}!({#4})$);
            \node[msgcircle, yshift=5.0mm] at (anchor) {#6};
            \node[yshift=2.0mm] at (anchor) {$\rightarrow$};
      }{}

      \ifthenelse{\isin{#1}{down} \AND \isin{#2}{left}}{
            \coordinate (anchor) at ($({#3})!{#5}!({#4})$);
            \node[msgcircle, yshift=-5.0mm] at (anchor) {#6};
            \node[yshift=-2.0mm] at (anchor) {$\leftarrow$};
      }{}
      \ifthenelse{\isin{#1}{up} \AND \isin{#2}{left}}{
            \coordinate (anchor) at ($({#3})!{#5}!({#4})$);
            \node[msgcircle, yshift=5.0mm] at (anchor) {#6};
            \node[yshift=2.0mm] at (anchor) {$\leftarrow$};
      }{}

      \ifthenelse{\isin{#1}{left} \AND \isin{#2}{up}}{
            \coordinate (anchor) at ($({#3})!{#5}!({#4})$);
            \node[msgcircle, xshift=-5.0mm] at (anchor) {#6};
            \node[xshift=-1.5mm] at (anchor) {$\uparrow$};
      }{}
      \ifthenelse{\isin{#1}{right} \AND \isin{#2}{up}}{
            \coordinate (anchor) at ($({#3})!{#5}!({#4})$);
            \node[msgcircle, xshift=5.0mm] at (anchor) {#6};
            \node[xshift=1.5mm] at (anchor) {$\uparrow$};
      }{}
}

\newcommand{\darkmsg}[6]{
      \ifthenelse{\isin{#1}{left} \AND \isin{#2}{down}}{
            \coordinate (anchor) at ($({#3})!{#5}!({#4})$);
            \node[darkmsgcircle, xshift=-5mm] at (anchor) {#6};
            \node[xshift=-1.5mm] at (anchor) {$\downarrow$};
      }{}
      \ifthenelse{\isin{#1}{right} \AND \isin{#2}{down}}{
            \coordinate (anchor) at ($({#3})!{#5}!({#4})$);
            \node[darkmsgcircle, xshift=5mm] at (anchor) {#6};
            \node[xshift=1.5mm] at (anchor) {$\downarrow$};
      }{}

      \ifthenelse{\isin{#1}{down} \AND \isin{#2}{right}}{
            \coordinate (anchor) at ($({#3})!{#5}!({#4})$);
            \node[darkmsgcircle, yshift=-5.0mm] at (anchor) {#6};
            \node[yshift=-2.0mm] at (anchor) {$\rightarrow$};
      }{}
      \ifthenelse{\isin{#1}{up} \AND \isin{#2}{right}}{
            \coordinate (anchor) at ($({#3})!{#5}!({#4})$);
            \node[darkmsgcircle, yshift=5.0mm] at (anchor) {#6};
            \node[yshift=2.0mm] at (anchor) {$\rightarrow$};
      }{}

      \ifthenelse{\isin{#1}{down} \AND \isin{#2}{left}}{
            \coordinate (anchor) at ($({#3})!{#5}!({#4})$);
            \node[darkmsgcircle, yshift=-5.0mm] at (anchor) {#6};
            \node[yshift=-2.0mm] at (anchor) {$\leftarrow$};
      }{}
      \ifthenelse{\isin{#1}{up} \AND \isin{#2}{left}}{
            \coordinate (anchor) at ($({#3})!{#5}!({#4})$);
            \node[darkmsgcircle, yshift=5.0mm] at (anchor) {#6};
            \node[yshift=2.0mm] at (anchor) {$\leftarrow$};
      }{}

      \ifthenelse{\isin{#1}{left} \AND \isin{#2}{up}}{
            \coordinate (anchor) at ($({#3})!{#5}!({#4})$);
            \node[darkmsgcircle, xshift=-5.0mm] at (anchor) {#6};
            \node[xshift=-1.5mm] at (anchor) {$\uparrow$};
      }{}
      \ifthenelse{\isin{#1}{right} \AND \isin{#2}{up}}{
            \coordinate (anchor) at ($({#3})!{#5}!({#4})$);
            \node[darkmsgcircle, xshift=5.0mm] at (anchor) {#6};
            \node[xshift=1.5mm] at (anchor) {$\uparrow$};
      }{}
}

\newcommand{\bwmsg}[6]{
      \ifthenelse{\isin{#1}{left} \AND \isin{#2}{down}}{
            \coordinate (anchor) at ($({#3})!{#5}!({#4})$);
            \node[msgdoublecircle, xshift=-5.5mm] at (anchor) {#6};
            \node[xshift=-1.5mm] at (anchor) {$\downarrow$};
      }{}
      \ifthenelse{\isin{#1}{right} \AND \isin{#2}{down}}{
            \coordinate (anchor) at ($({#3})!{#5}!({#4})$);
            \node[msgdoublecircle, xshift=5.5mm] at (anchor) {#6};
            \node[xshift=1.5mm] at (anchor) {$\downarrow$};
      }{}

      \ifthenelse{\isin{#1}{down} \AND \isin{#2}{right}}{
            \coordinate (anchor) at ($({#3})!{#5}!({#4})$);
            \node[msgdoublecircle, yshift=-6.0mm] at (anchor) {#6};
            \node[yshift=-2.0mm] at (anchor) {$\rightarrow$};
      }{}
      \ifthenelse{\isin{#1}{up} \AND \isin{#2}{right}}{
            \coordinate (anchor) at ($({#3})!{#5}!({#4})$);
            \node[msgdoublecircle, yshift=6.0mm] at (anchor) {#6};
            \node[yshift=2.0mm] at (anchor) {$\rightarrow$};
      }{}

      \ifthenelse{\isin{#1}{down} \AND \isin{#2}{left}}{
            \coordinate (anchor) at ($({#3})!{#5}!({#4})$);
            \node[msgdoublecircle, yshift=-6.0mm] at (anchor) {#6};
            \node[yshift=-2.0mm] at (anchor) {$\leftarrow$};
      }{}
      \ifthenelse{\isin{#1}{up} \AND \isin{#2}{left}}{
            \coordinate (anchor) at ($({#3})!{#5}!({#4})$);
            \node[msgdoublecircle, yshift=6.0mm] at (anchor) {#6};
            \node[yshift=2.0mm] at (anchor) {$\leftarrow$};
      }{}

      \ifthenelse{\isin{#1}{left} \AND \isin{#2}{up}}{
            \coordinate (anchor) at ($({#3})!{#5}!({#4})$);
            \node[msgdoublecircle, xshift=-5.5mm] at (anchor) {#6};
            \node[xshift=-1.5mm] at (anchor) {$\uparrow$};
      }{}
      \ifthenelse{\isin{#1}{right} \AND \isin{#2}{up}}{
            \coordinate (anchor) at ($({#3})!{#5}!({#4})$);
            \node[msgdoublecircle, xshift=5.5mm] at (anchor) {#6};
            \node[xshift=1.5mm] at (anchor) {$\uparrow$};
      }{}
}

\newcommand{\bwdarkmsg}[6]{
      \ifthenelse{\isin{#1}{left} \AND \isin{#2}{down}}{
            \coordinate (anchor) at ($({#3})!{#5}!({#4})$);
            \node[darkmsgdoublecircle, xshift=-5.5mm] at (anchor) {#6};
            \node[xshift=-1.5mm] at (anchor) {$\downarrow$};
      }{}
      \ifthenelse{\isin{#1}{right} \AND \isin{#2}{down}}{
            \coordinate (anchor) at ($({#3})!{#5}!({#4})$);
            \node[darkmsgdoublecircle, xshift=5.5mm] at (anchor) {#6};
            \node[xshift=1.5mm] at (anchor) {$\downarrow$};
      }{}

      \ifthenelse{\isin{#1}{down} \AND \isin{#2}{right}}{
            \coordinate (anchor) at ($({#3})!{#5}!({#4})$);
            \node[darkmsgdoublecircle, yshift=-6.0mm] at (anchor) {#6};
            \node[yshift=-2.0mm] at (anchor) {$\rightarrow$};
      }{}
      \ifthenelse{\isin{#1}{up} \AND \isin{#2}{right}}{
            \coordinate (anchor) at ($({#3})!{#5}!({#4})$);
            \node[darkmsgdoublecircle, yshift=6.0mm] at (anchor) {#6};
            \node[yshift=2.0mm] at (anchor) {$\rightarrow$};
      }{}

      \ifthenelse{\isin{#1}{down} \AND \isin{#2}{left}}{
            \coordinate (anchor) at ($({#3})!{#5}!({#4})$);
            \node[darkmsgdoublecircle, yshift=-6.0mm] at (anchor) {#6};
            \node[yshift=-2.0mm] at (anchor) {$\leftarrow$};
      }{}
      \ifthenelse{\isin{#1}{up} \AND \isin{#2}{left}}{
            \coordinate (anchor) at ($({#3})!{#5}!({#4})$);
            \node[darkmsgdoublecircle, yshift=6.0mm] at (anchor) {#6};
            \node[yshift=2.0mm] at (anchor) {$\leftarrow$};
      }{}

      \ifthenelse{\isin{#1}{left} \AND \isin{#2}{up}}{
            \coordinate (anchor) at ($({#3})!{#5}!({#4})$);
            \node[darkmsgdoublecircle, xshift=-5.5mm] at (anchor) {#6};
            \node[xshift=-1.5mm] at (anchor) {$\uparrow$};
      }{}
      \ifthenelse{\isin{#1}{right} \AND \isin{#2}{up}}{
            \coordinate (anchor) at ($({#3})!{#5}!({#4})$);
            \node[darkmsgdoublecircle, xshift=5.5mm] at (anchor) {#6};
            \node[xshift=1.5mm] at (anchor) {$\uparrow$};
      }{}
}

\makeatletter
\DeclareRobustCommand{\cev}[1]{%
  \mathpalette\do@cev{#1}%
}
\newcommand{\do@cev}[2]{%
  \fix@cev{#1}{+}%
  
\reflectbox{$\m@th#1\vec{\reflectbox{$\fix@cev{#1}{-}\m@th#1#2\fix@cev{#1}{+}$}}$}%
  \fix@cev{#1}{-}%
}
\newcommand{\fix@cev}[2]{%
  \ifx#1\displaystyle
    \mkern#23mu
  \else
    \ifx#1\textstyle
      \mkern#23mu
    \else
      \ifx#1\scriptstyle
        \mkern#22mu
      \else
        \mkern#22mu
      \fi
    \fi
  \fi
}

\usetikzlibrary{arrows.meta}
\usetikzlibrary{backgrounds}
\usepgfplotslibrary{patchplots}
\usepgfplotslibrary{fillbetween}
\pgfplotsset{%
    layers/standard/.define layer set={%
        background,axis background,axis grid,axis ticks,axis lines,axis tick labels,pre main,main,axis descriptions,axis foreground%
    }{
        grid style={/pgfplots/on layer=axis grid},%
        tick style={/pgfplots/on layer=axis ticks},%
        axis line style={/pgfplots/on layer=axis lines},%
        label style={/pgfplots/on layer=axis descriptions},%
        legend style={/pgfplots/on layer=axis descriptions},%
        title style={/pgfplots/on layer=axis descriptions},%
        colorbar style={/pgfplots/on layer=axis descriptions},%
        ticklabel style={/pgfplots/on layer=axis tick labels},%
        axis background@ style={/pgfplots/on layer=axis background},%
        3d box foreground style={/pgfplots/on layer=axis foreground},%
    },
}

\makeatletter
\let\c@theorem\@undefined
\let\theorem\@undefined \let\endtheorem\@undefined
\let\c@lemma\@undefined
\let\lemma\@undefined \let\endlemma\@undefined
\let\c@remark\@undefined
\let\remark\@undefined \let\endremark\@undefined
\let\c@proposition\@undefined
\let\proposition\@undefined \let\endproposition\@undefined
\makeatother
\newtheorem{theorem}{Theorem}
\newtheorem{lemma}{Lemma}
\newtheorem{remark}{Remark}
\newtheorem{proposition}{Proposition}
\crefname{theorem}{Theorem}{Theorems}
\Crefname{theorem}{Theorem}{Theorems}
\crefname{lemma}{Lemma}{Lemmas}
\Crefname{lemma}{Lemma}{Lemmas}
\crefname{remark}{Remark}{Remarks}
\Crefname{remark}{Remark}{Remarks}
\crefname{proposition}{Proposition}{Propositions}
\Crefname{proposition}{Proposition}{Propositions}
\crefname{figure}{Figure}{Figures}

\newcommand{\reals}{\mathbb{R}}

\newcommand{\bigO}{\mathcal{O}}

\DeclareMathOperator{\blkdiag}{blkdiag}

\newcommand{\dif}{\mathop{}\!\mathrm{d}}
\newcommand\given[1][]{\,#1\vert\,}

\jmlrvolume{}
\jmlryear{2026}
\jmlrworkshop{Probabilistic Graphical Models (PGM)}

\title{A Factor Graph Approach to Scalable Multi-Output Gaussian Process Regression}

\author{
\Name{Wouter W. L. Nuijten} \Email{w.w.l.nuijten@tue.nl}\\
\addr Eindhoven University of Technology \& Lazy Dynamics B.V., the Netherlands
\AND
 \Name{Esther G. {van Pelt}} \Email{e.g.v.pelt@tue.nl}\\
 \addr Eindhoven University of Technology, Eindhoven, the Netherlands
 \AND
  \Name{Albert Podusenko} \Email{albert@lazydynamics.com} \\
 \Name{{\.{I}}smail \c{S}en\"{o}z} \Email{isenoz@lazydynamics.com} \\
  \addr Lazy Dynamics B.V., the Netherlands
  \AND
 \Name{Wouter M. Kouw} \Email{w.m.kouw@tue.nl}\\
  \addr Eindhoven University of Technology, Eindhoven, the Netherlands
  }

\editor{Gustau Camps-Valls, Manuele Leonelli and Gherardo Varando}

\begin{document}

\maketitle

\begin{abstract}
  Multi-output Gaussian process regression scales cubically in the number of observations times outputs, and dense kernel-matrix methods need bespoke handling whenever different outputs are observed at different inputs.
  We express multi-output Gaussian process regression as a Forney-style factor graph in which a nearest-neighbor chain orders a fixed candidate set of $C$ inputs into a one-dimensional sequence.
  Along this chain, latent Matérn processes evolve through linear-Gaussian transition factors, while the linear model of coregionalization mixes $L$ latent processes into $D$ outputs through a deterministic mixing factor and per-output scalar observation factors.
  Posterior computation reduces to exact Gaussian message passing on the chain at cost $\bigO(C(DL^2 + L^3))$ after chain construction, and missing observations omit their local factor without any covariance-matrix restructuring.
  The formulation therefore scales in the number of data samples and in the rate of missing observations, while remaining best suited to candidate sets in low input dimension.
  We compare the factor-graph formulation against an exact kernel-matrix baseline, a sparse-variational inducing-point baseline, and a nearest-neighbor baseline on a synthetic input-dimension sweep and on electricity time series forecasting.
  At low input dimension the factor-graph posterior tracks the exact kernel-matrix posterior closely, and the gap grows gradually as input dimension increases while staying competitive with both approximate baselines.
  On the electricity time series our factor-graph formulation matches all three baselines in forecast accuracy while scaling linearly in the number of data points, where the exact kernel-matrix method becomes infeasible and the inducing-point baseline remains substantially slower.
\end{abstract}

\begin{keywords}
  Factor graphs, Gaussian processes, Message passing, Partial observations
\end{keywords}

\section{Introduction}
\label{sec:introduction}

Probabilistic graphical models provide a unified language for inference: belief propagation, Kalman filtering and smoothing, hidden Markov model inference, and variational message passing all arise as instances of message passing on specific factor graphs \citep{loeliger_factor_2007,senoz_variational_2021}.
Gaussian processes (GPs) are a notable outlier.
Defined by a kernel rather than a factorization, GPs are typically treated as monolithic objects, which do not scale well and have poor compositionality.
In the multi-output case, fitting a standard GP with $D$ correlated outputs and $N$ training points requires operations on a $DN \times DN$ covariance matrix, i.e., $\bigO((DN)^3)$, which becomes prohibitive as $D$ or $N$ grows \citep{alvarez_kernels_2012}.
Equally inconvenient, when different outputs are observed at different inputs, as is the norm in sensor networks, environmental monitoring, and clinical time series, dense kernel-matrix implementations typically handle these partial observations by restructuring or slicing the covariance matrix rather than making a local edit to a graph.

But non-monolithic treatments of GPs do exist.
A GP with a Mat{\'e}rn covariance function can be expressed as the solution of a linear stochastic differential equation (SDE) \citep{hartikainen_kalman_2010,sarkka_applied_2019}, and the resulting linear-Gaussian state-space model has an exact representation as a Forney-style factor graph (FFG) on which Kalman filtering and smoothing reduce to message passing \citep{loeliger_factor_2007}.
This bridge has been widely exploited in temporal settings, where the natural ordering of time supplies the one-dimensional chain that the state-space form requires.
We use this bridge as the starting point for multi-output regression: we cast low-input-dimensional multi-output GP (MOGP) regression as message passing on an FFG by combining state-space GPs along a greedy nearest-neighbor chain over a fixed candidate set of size $C$ with the Linear Model of Coregionalization (LMC), expressed as a deterministic mixing factor and per-output scalar observation factors.
The chain construction is an approximation, since $M$-dimensional input geometry is compressed onto a one-dimensional Markov sequence. Here we investigate the fidelity of such a compression.

Casting the model in this form has two consequences for inference.
First, conditional on the chain, posterior computation is exact Gaussian message passing: the $L$ latent dynamics remain block diagonal in the prior, while the LMC observation factors couple them into dense Gaussian beliefs over the joint latent state.
For fixed $D$ and $L$, inference is therefore linear in the candidate-chain length $C$, with $\bigO(C(DL^2 + L^3))$ cost after chain construction.
Second, partial observability is native to the factor-graph implementation: outputs that are unobserved at a given input simply omit their local observation message, with no covariance-matrix restructuring and no change to the model or to the inference call.
Dense kernel-matrix implementations, by contrast, must explicitly build or slice variable-size covariance matrices from the observed (output, input) pairs.

Concretely, our contribution is firstly a chain-induced factor graph formulation of low-input-dimensional multi-output GP regression, combining state-space GPs \citep{hartikainen_kalman_2010}, FFG message passing \citep{loeliger_factor_2007}, and the LMC \citep{alvarez_kernels_2012}, with $\bigO(C(DL^2 + L^3))$ per-step inference after chain construction. Approximation occurs only through the nearest-neighbor chain, inference on the resulting model is exact. Secondly, partial observability requires no restructuring of the inference procedure, as opposed to dense kernel-matrix treatments that require covariance-matrix slicing.

\section{Problem Statement} \label{sec:problem}

Let $x \in \mathcal{X} \subset \reals^M$ denote an input and $f : \mathcal{X} \to \reals^D$ a vector-valued function with $D$ correlated outputs.
Given a training set of $N$ inputs $X \in \reals^{N \times M}$ and a noisy observation matrix $Y \in \reals^{N \times D}$ with $y_i = f(x_i) + \epsilon_i$ and Gaussian noise $\epsilon_i \sim \mathcal{N}(0, \Sigma_\epsilon)$, where $\Sigma_\epsilon \in \reals^{D \times D}$ is a diagonal observation-noise covariance, we wish to infer the posterior predictive distribution of $f$ at unseen inputs under a zero-mean multi-output Gaussian process prior.
A binary mask $O \in \{0,1\}^{N \times D}$ records which entries are observed, with arbitrary missingness allowed across $(i,d)$.

We target exact Gaussian inference under the chain-induced factor-graph model with two requirements.
First, inference cost should grow linearly in the number of training points $N$ for fixed output count $D$ and fixed latent rank.
Second, the missingness mask $O$ should be absorbed into the model itself, so that arbitrary missingness patterns produce no additional restructuring cost.
Kernel-matrix methods satisfy neither requirement, as vectorized inference costs $\bigO((ND)^3)$ and varying missingness forces a variable-size covariance to be rebuilt from the observed entries.
\Cref{sec:method} presents a factor-graph reformulation that meets both requirements, and \cref{sec:evaluation} assesses posterior quality empirically.

\section{Background}
\label{sec:related_work}

% \subsection{Bayesian Optimization}
% \label{sec:rw_bo}

\subsection{Multi-Output Gaussian Processes}
% \subsection{Linear Model of Coregionalization}
\label{sec:rw_gp}
% There are a number of approaches to generalizing BO to multi-dimensional functions. 
% Scalarization approaches such as ParEGO reduce multi-objective problems to single-objective ones via random weight vectors \citep{knowles_parego_2006}, while hypervolume-based methods optimize the Pareto front directly \citep{daulton_differentiable_2020,daulton_parallel_2021}.
% Information-theoretic acquisition functions extend to the multi-objective setting as well \citep{hernandez-lobato_predictive_2016}.
% Multi-task BO shares surrogate information across correlated tasks \citep{swersky_multitask_2013}.
% Our work uses a scalarization-based UCB, but the factor graph formulation is agnostic to the choice of acquisition function.

Exact multi-output GP inference is costly because covariance matrices couple inputs and outputs. We focus on the Linear Model of Coregionalization (LMC), which mixes $L$ latent GPs $g_l(x)$ through a mixing matrix $W \in \mathbb{R}^{D \times L}$ with induced cross-output covariance
\begin{align}
  f_d(x) = \sum_{l=1}^L w_{dl} g_l(x) \, , \qquad   \operatorname{Cov}\!\left[f_d(x), f_{d'}(x')\right] = \sum_{l=1}^{L} w_{dl} w_{d'l} k_l(x,x') \, .
\end{align}
At a fixed input, or with shared latent input covariance, the output-side structure is low-rank plus diagonal noise. For $L \ll D$, Woodbury applies inverses of $W S W^\top + \Sigma_\epsilon$ through an $L \times L$ inner system, costing $\bigO(DL^2 + L^3)$, or $\bigO(DL^2)$ for fixed~$L$, instead of dense $D \times D$ operations. The special case $L=1$ is the Intrinsic Coregionalization Model (ICM), which is less expressive than LMC \citep{alvarez_computationally_2011,bonilla_multitask_2007}.
% \citet{bruinsma_scalable_2020} exploit orthogonal mixing matrices for exact multi-output inference at cost linear in the number of latent processes.
% \esther{Het lijkt alsof deze zin hieronder het nut van de hele alinea uit moet leggen, maar ik vind het niet heel duidelijk zo.}
% Streaming sparse GP approximations enable online surrogate updates as new data arrives \citep{bui_streaming_2017}.
% \wmk{todo me: explain LMC}

\subsection{State-Space Gaussian Processes}
\label{sec:ss_gp}

The cubic complexity in samples has driven extensive work on scalable approximations.
Sparse variational methods based on inducing points reduce this cost while preserving a variational lower bound on the marginal likelihood \citep{smola_laplace_2004,bruinsma_scalable_2020}.
However, these are approximations. Temporal GPs with Mat\'ern covariance functions admit an exact representation as linear SDEs, converting GP regression into Kalman filtering and smoothing at $\bigO(N)$ cost \citep{hartikainen_kalman_2010,sarkka_bayesian_2013,sarkka_applied_2019}.
% \esther{ik vind dit een beetje random hier:}
% This equivalence extends to non-Gaussian likelihoods via approximate message passing schemes such as expectation propagation and variational inference \citep{nickisch_state_2018,wilkinson_state_2020}.
% Efficient banded linear algebra further enables gradient-based hyperparameter learning within the state-space framework \citep{durrande_banded_2019}.
%
For example, a scalar temporal GP with a stationary Mat\'ern $3/2$ covariance function $\kappa(t,t')$, length-scale $\ell$, and output scale $\gamma^2$ can be represented exactly as the solution of the linear SDE
% \begin{equation}
%   \label{eq:sde}
%   \dif\bm{x}(t) = F\,\bm{x}(t)\,\dif t + L\,\dif\bm{w}(t)\,,
% \end{equation}
% where $\bm{x}(t) = [f(t),\, f'(t)]^\top$ is the two-dimensional state, $L = [0,\, 1]^\top$ is the diffusion matrix, and
% \begin{equation}
%   \label{eq:F_matrix}
%   F = \begin{bmatrix} 0 & 1 \\ -\lambda^2 & -2\lambda \end{bmatrix}\,,
%   \quad \lambda = \frac{\sqrt{3}}{\ell}\,.
% \end{equation}
\begin{align}
  \dif\begin{bmatrix} g(t) \\ \dot{g}(t) \end{bmatrix} = \begin{bmatrix} 0 & 1 \\ -\lambda^2 & -2\lambda \end{bmatrix} \begin{bmatrix} g(t) \\ \dot{g}(t) \end{bmatrix} \dif t + \begin{bmatrix} 0 \\ 1 \end{bmatrix} \dif w(t) \, ,
\end{align}
for $\lambda = \sqrt{3} / \ell$ and white noise $w(t)$ with spectral density $4\lambda^3 \gamma^2$, where $\gamma^2$ is the GP output scale \citep{hartikainen_kalman_2010}. We shall refer to the state vector of the GP $g(t)$ and its derivative in time $\dot{g}(t)$ as $z(t)$.
The stationary covariance of the state is $P_\infty = \text{diag}(\begin{bmatrix} \gamma^2 &  \lambda^2 \gamma^2 \end{bmatrix})$.
%\begin{equation}
%  \label{eq:P_inf}
%  P_\infty = \begin{bmatrix} \gamma^2 & 0 \\ 0 & \lambda^2 \gamma^2 \end{bmatrix}\,.
%\end{equation}
This SDE can be discretized to form a discrete-time state-space model over states $z_t = [g_t, \, \dot{g}_t]^{\top}$.
Given the inter-point distance $\Delta_i$ between consecutive points in the chain, the discrete-time transition and process noise matrices are
\begin{equation}
  \label{eq:discrete_transition}
  A_i \! = \! \exp \left (\begin{bmatrix} 0 & 1 \\ -\lambda^2 & -2\lambda \end{bmatrix} \Delta_i \right )\, ,
  \qquad  Q_i \! = \! P_\infty \! - \! A_i\, P_\infty\, A_i^\top .
\end{equation}
The observation model extracts the function value from the state through the observation vector $h = [1,\, 0]^\top$, so that $g_t = h^\top z_t$.
This yields a linear-Gaussian state-space model equivalent to the GP, with $\bigO(N)$ inference in the number of points for a single scalar output via Bayesian filtering and smoothing \citep{sarkka_bayesian_2013}. The multi-output cost is treated in \cref{sec:inference}.

\subsection{Factor Graphs and Message Passing}
\label{sec:rw_fg}

Factor graphs provide a unified language for probabilistic inference, subsuming belief propagation, variational inference and expectation propagation as message computation on a graph \citep{kschischang_factor_2001,senoz_variational_2021,heess_learning_2013}.
Many classical algorithms, including Kalman filtering and smoothing, can be written as message passing on particular factor graphs \citep{loeliger_factor_2007,palmieri_unifying_2022}.
Take, for example, the below probabilistic model over $x_1, x_2$ and $x_3$, and assume that we are interested in the marginal distribution for $x_3$:
\begin{equation}
  p(x_1, x_2, x_3) = p(x_1 \given x_2) p(x_2 \given x_3) p(x_3) \, , \qquad p(x_3) = \int p(x_1, x_2, x_3) \mathrm{d}(x_1,x_2) \, .
\end{equation}
In a Forney-style factor graph (FFG), factors are nodes and variables are edges. Each edge connects at most two factor nodes \citep{loeliger_introduction_2004}.
The notation is illustrated in \cref{fig:FFGexample}: boxes denote factors, edge labels denote variables, and arrows label messages passed along those variables.
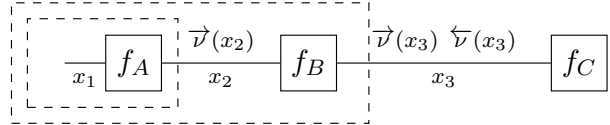
\begin{figure}[!htb]
  \centering
  \resizebox{.58\textwidth}{!}{\begin{tikzpicture}

    \node [box] (f_A) {$f_A$};
    \node [box, right of=f_A, node distance=22mm] (f_B) {$f_B$};
    \node [box, right of=f_B, node distance=34mm] (f_C) {$f_C$};
    \node [left of=f_A, node distance=10mm] (half_A) {};
    \node [right of=f_C, node distance=10mm] (half_C) {};

    \draw[] (half_A) -- node[below] {$\scriptstyle{x_1}$} (f_A);
    \draw[] (f_A) -- node[below] {$\scriptstyle{x_2}$} node[above] {$\scriptstyle{\overrightarrow{\nu}(x_2)}$} (f_B);
    \draw[] (f_B) -- node[below] {$\scriptstyle{x_3}$} node[above] {$\scriptstyle{\overrightarrow{\nu}(x_3)} \ \scriptstyle{\overleftarrow{\nu}(x_3)}$} (f_C);
    % \draw[] (f_C) -- node[below] {$\scriptstyle{x_4}$} (half_C);

    \node[dashed, fit=(half_A)(f_A), draw, inner sep=2mm] {};
    \node[dashed, fit=(half_A)(f_B), draw, inner sep=4mm] {};

\end{tikzpicture}}
  \caption{Example FFG. Messages along an edge summarize all factors of other edges.}
  \label{fig:FFGexample}
\end{figure}
For a generic node $f(y, x_1, \dots, x_{n})$, the sum-product message to edge $y$ is
\begin{equation}\label{sprule}
  \overrightarrow{\nu}(y) = \int \dots \int f(y, x_1, \dots, x_{n}) \prod_{i=1}^{n} \overrightarrow{\nu}(x_i)\mathrm{d}x_i \, ,
\end{equation}
where $\overrightarrow{\nu}(x_i)$ are incoming messages to the node.
Priors can be represented as terminal factors, likelihoods send observation information back into the graph, and posterior marginals are obtained by multiplying incoming messages on an edge \citep{senoz_variational_2021}.
% Variational message passing is a variant where one defines a Bethe free energy cost functional to obtain approximate marginal distributions \citep{senoz_message_2022}. Minimizing the Bethe free energy produces a message computation rule in the form:
% \begin{align}
%     \overrightarrow{\nu}(y) = \exp \int \dots \int \prod_{i=1}^{n} \overrightarrow{\nu}(x_i) \ln f(y,x_1, \dots, x_n) \mathrm{d}x_i \, .
% \end{align}
%
A key property of FFGs is compositionality: model changes, such as adding or removing observation factors, require only local graph modifications, not a new inference algorithm.
Reactive message passing exploits this structure by compiling model declarations into message updates and recomputing only affected messages when data or factors change \citep{bagaev_reactive_2023}.
This compositionality is central to our treatment of partial observations.

\section{Message Passing-Based Multi-Output GP Inference}
\label{sec:method}

%This section describes the proposed inference framework. We specify a multi-output state-space GP in factor graph form (\cref{sec:model_spec}) and explain how to infer states under partial observation (\cref{sec:inference}), yielding the posterior predictive distribution at every candidate input.
% First, we review the state-space representation of GPs with Mat\'ern covariance functions (\cref{sec:ss_gp}).  
% We then combine this with the LMC for multiple outputs and per-output scalar observation noise (\cref{sec:lmc}).
% \Cref{sec:partial_obs} describes how the factor graph formulation handles partial observations.

\subsection{Model Specification} \label{sec:model_spec}
%We first describe a nearest neighbor chain-based ordering for fixed low-dimensional candidate sets (\cref{sec:nn_ordering}), and then construct a factor graph based on a state-space GP with LMC (\cref{sec:ssp+lmc}).

\subsubsection{Nearest-Neighbor Chain Ordering}
\label{sec:nn_ordering}

Extending state-space GPs beyond one-dimensional inputs is an active research area.
In spatiotemporal settings, a natural temporal ordering exists, and the spatial structure can be handled separately via inducing points \citep{sarkka_spatiotemporal_2013,tebbutt_combining_2021,hamelijnck_spatiotemporal_2021}.
An alternative approach is the SPDE method of \citet{lindgren_explicit_2011}, which discretizes the stochastic partial differential equation on a mesh to obtain a sparse precision matrix for Mat\'ern fields in arbitrary dimensions.
Recently, \citet{li_learning_2025} proposed a universal method for converting any stationary temporal GP into a state-space model via spectral factorization.
Our work differs from these spatiotemporal approaches in that we do not assume separable structure. Instead, we construct a one-dimensional chain via nearest-neighbor ordering. This makes state-space inference applicable to non-temporal candidate sets, but it is an approximation whose fidelity depends on how well the chain preserves the relevant geometry.
A closely related family also relies on nearest-neighbor orderings: the Vecchia approximation \citep{vecchia_estimation_1988} conditions each point on a small set of preceding neighbors to obtain a sparse Cholesky factor of the precision matrix, realized as the hierarchical nearest-neighbor GP of \citet{datta_hierarchical_2016} and generalized by \citet{katzfuss_general_2021}. These methods produce an \emph{approximate} sparse precision from variable-size conditioning sets, whereas our Markov-1 chain gives an exact message-passing recursion at the cost of compressing the $M$-dimensional geometry into a single chain.

% State-space GPs require a one-dimensional ordering of the data points.
For $M$-dimensional inputs $X \in \reals^{C \times M}$, with fixed candidate-set size $C$, we construct this ordering via a greedy nearest-neighbor heuristic. Let $\pi : \{1,\ldots,C\} \to \{1,\ldots,C\}$ denote the resulting permutation, so that $x_{\pi(i)}$ is the $i$-th point visited in the chain: starting from an arbitrary point (e.g., $x_1$), we repeatedly visit the closest unvisited point in Euclidean distance and record the inter-point distances $\Delta_i = \|x_{\pi(i)} - x_{\pi(i-1)}\|$, which serve as the time steps in the state-space model.
%Points close in the original input space have small $\Delta_i$ (tight coupling), while distant points have large $\Delta_i$ (weak coupling, approaching the prior).
%
This ordering defines a chain-induced kernel approximation.
Let $s_{\pi(1)} = 0$ and $s_{\pi(i)} = \sum_{r=2}^{i} \Delta_r$ be the one-dimensional coordinate induced by the chain.
The chain distance is the distance between these coordinates and the chain kernel is the kernel defined by the Matérn kernel covariance function evaluated at the chain distance:
\begin{equation} \label{eq:K_pi}
  \left[K_{\pi} \right]_{ij} = \kappa\big( | s_{\pi(i)} - s_{\pi(j)} | \big) \, .
\end{equation}
A naive greedy approach to constructing the chain will cost $\bigO(C^2)$, but techniques such as $kd$-trees can speed this up to $\bigO(C\log C)$ in low input dimension \citep{cormen_introduction_2022}.

To characterize the fidelity of the chain-induced kernel approximation, we derive a bound on the distance in Matérn-kernel GP posterior parameters.
First, we tackle the case of single-output GP, i.e., for the $l$-th GP, and then generalize to multi-output, i.e., all $L$ GPs.
\begin{lemma} \label{lem:chain_kernel_bound}
  Let $\kappa_l(\cdot)$ be a Matérn-class kernel with length-scale $\ell_l > 0$ and output scale $\gamma_l^2 > 0$, and let $\alpha_l$ be its Lipschitz constant on the range of distances induced by $X$ and the chain $\pi$. Define $\left[K_l\right]_{ij} = \kappa_l\big(\|x_i - x_j\|\big)$, $\left[K_{\pi,l}\right]_{ij} = \kappa_l\big(|s_{\pi(i)} - s_{\pi(j)}|\big)$ and $\varepsilon_{\pi}
    = \max_{i,j}\, \bigl|\|x_i-x_j\| - |s_{\pi(i)} - s_{\pi(j)}| \bigr|$.
  Then,
  \begin{equation}
    \|K_l - K_{\pi,l}\|_2 \leq C\,\alpha_l\,\varepsilon_{\pi} \, .
  \end{equation}
\end{lemma}
The proof is in \cref{app:lemma1}. Here $C$ is the candidate-set size, so the bound scales with the number of points but vanishes as the chain distortion $\varepsilon_\pi \to 0$, that is, as the one-dimensional ordering better preserves the input geometry.
Under scalar observation noise of variance $\sigma_n^2$, the GP posterior parameters obtained by conditioning on observations $y$ are $\mu = K(K + \sigma_n^2 I)^{-1}y$ and $\Sigma = K - K(K + \sigma_n^2 I)^{-1}K$ \citep{rasmussen_gaussian_2006}, with $\mu_\pi$ and $\Sigma_\pi$ the same expressions using the chain-induced kernel of \cref{eq:K_pi} in place of $K$. A resolvent perturbation argument converts a bound on $\|K - K_\pi\|_2$ into bounds on these parameters. We defer that step (\cref{lem:posterior_perturbation} in \cref{app:lemma2}) and state the multi-output LMC form directly. Writing $w_l := W_{:,l} \in \reals^{D}$ for the $l$-th column of the mixing matrix, the Euclidean and chain-induced LMC covariance matrices are defined as
\begin{equation}
  K = \sum_{l=1}^{L} (w_l w_l^\top) \otimes K_l \, , \qquad
  K_{\pi} = \sum_{l=1}^{L} (w_l w_l^\top) \otimes K_{\pi,l} \, ,
\end{equation}
where $K_l, K_{\pi,l}$ are as in \cref{lem:chain_kernel_bound}.
\begin{theorem}
  \label{thm:lmc_posterior_perturbation}
  Let $\eta = \|K - K_{\pi}\|_2$. Assume scalar observation noise $\sigma_n^2 I$ and that each $\kappa_l$ is a Matérn-class kernel with parameters $(\ell_l, \gamma_l^2)$ and Lipschitz constant $\alpha_l$. Then the multi-output GP posterior parameters obtained by conditioning on output matrix $Y$ satisfy
  \begin{align}
    \|\mu - \mu_{\pi}\|_2       & \leq \left(\frac{\eta}{\sigma_n^2} + \frac{(\|K\|_2 + \eta)\,\eta}{\sigma_n^4}\right) \|Y\|_2 \, ,                    \\
    \|\Sigma - \Sigma_{\pi}\|_2 & \leq \eta + \frac{\eta\,(\|K\|_2 + \|K_{\pi}\|_2)}{\sigma_n^2} + \frac{\eta\,\|K\|_2\,\|K_{\pi}\|_2}{\sigma_n^4} \, ,
  \end{align}
  with $\eta \leq C\, \varepsilon_\pi \sum_{l=1}^{L} \|w_l\|_2^2\, \alpha_l$.
\end{theorem}
The proof is in \cref{app:thm1}. Thus, the difference between the per-latent bound and the total multi-output GP bound is the weighted sum of Lipschitz constants.
For inputs that are roughly uniform in $[0,1]^M$, the typical nearest-neighbor distance scales as $C^{-1/M}$, so the greedy chain has unnormalized path length $\sum_i \tilde\Delta_i = \Theta(C^{1-1/M})$. The chain coordinate range $\max_{i,j} |s_{\pi(i)} - s_{\pi(j)}|$ inherits this scaling, while the Euclidean diameter is bounded by $\sqrt{M}$. Hence the worst-case distortion satisfies $\varepsilon_\pi = O(C^{1-1/M})$ for $M \ge 2$, while at $M{=}1$ the sorted chain is an isometry of the Euclidean ordering and $\varepsilon_\pi = 0$. Plugging back into \cref{thm:lmc_posterior_perturbation}, $\eta = \mathcal{O}\bigl(C^{2-1/M} \sum_l \|w_l\|_2^2 \alpha_l\bigr)$. This is informative for small $M$, but becomes uninformative as $M$ grows.
%This matches the growing chain stretch and widening accuracy gap reported in \cref{sec:dim_sweep}. 
We stress that both bounds are worst case. $\varepsilon_\pi$ is a maximum over all pairs, and $\eta$ enters the posterior bounds through $\sigma_n^{-4}$. So at practical $C$ and small observation noise they are loose and certify no numerically small posterior error.
%They merely identify $M$ as the quantity that controls chain fidelity.

Different starts of the chain produce different permutations $\pi$ and hence different chain coordinate ranges. A classical result on the metric traveling salesman problem bounds this dependence.
\begin{proposition} \label{prop:chain_start}
  Define $T(\pi) := \max_i s_{\pi(i)}$ and let $\pi_1, \pi_2$ be two greedy nearest-neighbor chains over the same candidate set $X$ from different starting points. Then
  \begin{equation}
    T(\pi_1) \;\leq\; \tfrac{1}{2}\bigl(\lceil \log_2 C \rceil + 1\bigr)\,T(\pi_2) \, .
  \end{equation}
\end{proposition}
The proof is in \cref{app:prop1}. Since $\varepsilon_\pi \le T(\pi)$ for every chain, the bound of \cref{thm:lmc_posterior_perturbation} varies by at most an $\mathcal{O}(\log C)$ factor across starting points. \cref{app:chain_start} measures the variation that actually occurs, and finds it far inside this envelope.

\subsubsection{State-Space GP with LMC}
\label{sec:ssp+lmc}

In state-space form, the full state $z_i \in \reals^{2L}$ is the concatenation of all latent states.
The system matrices are block-diagonal:
\begin{align}
  \label{eq:block_diag}
  A_i  \!=\! \blkdiag\!\big(A_i^{(1)}, .., A_i^{(L)}\big) ,           \ \
  Q_i \! =\! \blkdiag\!\big(Q_i^{(1)}, .., Q_i^{(L)}\big),           \ \
  P  \! =\! \blkdiag\!\big(P_\infty^{(1)}, .. , P_\infty^{(L)}\big) .
\end{align}
The observation matrix combines the mixing matrix $W$ with the per-latent observation vectors $h^{(l)} = [1,\,0]^\top$ (each picking out the function value from the $l$-th latent state):
\begin{equation}
  \label{eq:H_matrix}
  H = \begin{bmatrix}
    W_{:,1}\, h^{(1)\top} & \cdots & W_{:,L}\, h^{(L)\top}
  \end{bmatrix} \in \reals^{D \times 2L}\, ,
\end{equation}
so that $H\,z_i$ yields a $D$-dimensional vector.
Each output dimension is observed independently with per-output scalar noise variance $\tau_d^{-1}$.
This per-output scalar formulation keeps the model fully linear-Gaussian, enabling exact single-pass inference without variational approximations.
% The full formulation preserves linear scaling in C for fixed D and L.
%
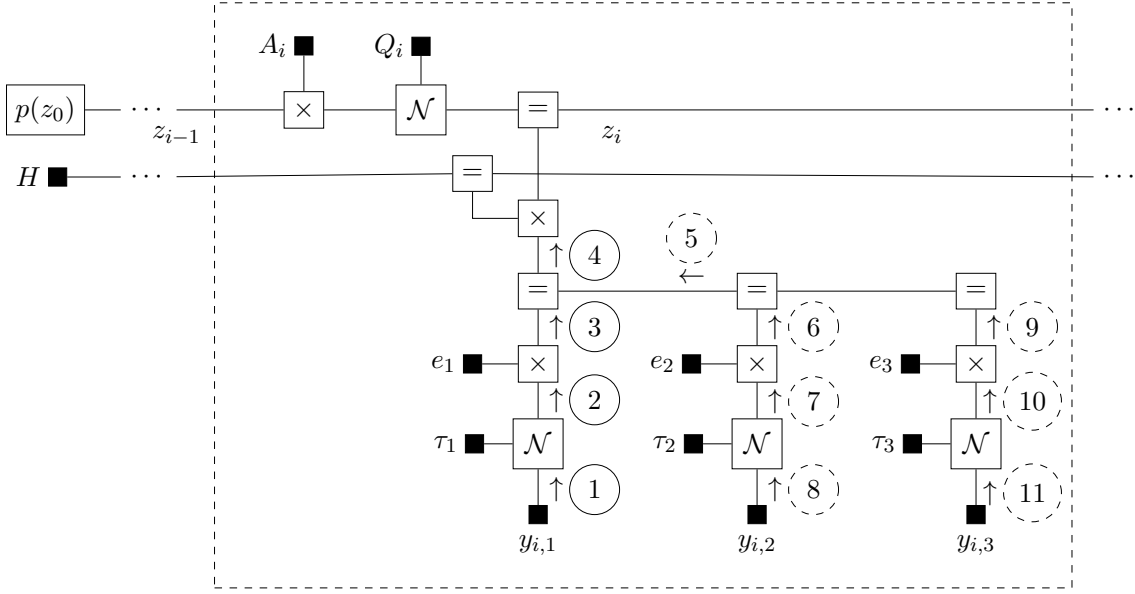
\begin{figure*}[tb]
  \centering
  %    \resizebox{!}{150pt}{\input{graphs/ffg_kmgp}} \hfill
  %    \resizebox{!}{150pt}{\input{graphs/ffg.tex}}
  \resizebox{\textwidth}{!}{\begin{tikzpicture}
    % nodes
        % top part
        \node[smallbox] (x) {$\times$};
        \node[blackbox, above=5mm of x] (box_at) {};
        \node[left=1mm of box_at, minimum size = 0 mm, inner sep = 0 mm, outer sep = 0 mm] (at) {$A_i$};
        \node[box, right=10mm of x] (normal_x) {$\mathcal{N}$};
        \node[blackbox, above=4mm of normal_x] (box_qt) {};
        \node[left=1mm of box_qt, minimum size = 0 mm, inner sep = 0 mm, outer sep = 0 mm] (qt) {$Q_i$};
        \node[smallbox, right=10mm of normal_x] (eq_x) {$=$};
        \node[smallbox, below=10mm of eq_x] (mult_my1) {$\times$};
        \node[left=5mm of eq_x] (anc_h) {};
        \node[smallbox, below=5mm of anc_h] (eq_h) {$=$};

        % anchors
        \node[left=15mm of x] (anc_prev) {$\cdots$};
        \node[right=75mm of eq_x] (anc_next) {$\cdots$};
        \node[below=5mm of anc_prev] (anc_hprev) {$\cdots$};
        \node[below=5mm of anc_next] (anc_hnext) {$\cdots$};
        
        % initial
        \node[box, left=5mm of anc_prev] (normal_x0) {$p(z_0)$};
        \node[blackbox, left=7.5mm of anc_hprev] (box_h0) {};
        \node[left=1mm of box_h0, minimum size = 0 mm, inner sep = 0 mm, outer sep = 0 mm] (h0) {$H$};
        
        % block yt1
        \node[smallbox, below=5mm of mult_my1] (eq_my1) {$=$};
        \node[smallbox, below=5mm of eq_my1] (mult_e1) {$\times$};
        \node[blackbox, left=5mm of mult_e1] (box_e1) {};
        \node[left=1mm of box_e1, minimum size = 0 mm, inner sep = 0 mm, outer sep = 0 mm] (e1) {$e_1$};
        \node[box, below=5mm of mult_e1] (normal_y1) {$\mathcal{N}$};
        \node[blackbox, left=4mm of normal_y1] (box_tau1) {};
        \node[left=1mm of box_tau1, minimum size = 0 mm, inner sep = 0 mm, outer sep = 0 mm] (tau1) {$\tau_1$};
        \node[blackbox, below=5mm of normal_y1] (box_yt1) {};
        \node[below=1.5mm of box_yt1, minimum size = 0 mm, inner sep = 0 mm, outer sep = 0 mm] (yt1) {$y_{i,1}$};

        % block yt2
        \node[smallbox, right=25mm of eq_my1] (eq_my2) {$=$};
        \node[smallbox, below=5mm of eq_my2] (mult_e2) {$\times$};
        \node[blackbox, left=5mm of mult_e2] (box_e2) {};
        \node[left=1mm of box_e2, minimum size = 0 mm, inner sep = 0 mm, outer sep = 0 mm] (e2) {$e_2$};
        \node[box, below=5mm of mult_e2] (normal_y2) {$\mathcal{N}$};
        \node[blackbox, left=4mm of normal_y2] (box_tau2) {};
        \node[left=1mm of box_tau2, minimum size = 0 mm, inner sep = 0 mm, outer sep = 0 mm] (tau2) {$\tau_2$};
        \node[blackbox, below=5mm of normal_y2] (box_yt2) {};
        \node[below=1.5mm of box_yt2, minimum size = 0 mm, inner sep = 0 mm, outer sep = 0 mm] (yt2) {$y_{i,2}$};
        
        % block yt3
        \node[smallbox, right=25mm of eq_my2] (eq_my3) {$=$};
        \node[smallbox, below=5mm of eq_my3] (mult_e3) {$\times$};
        \node[blackbox, left=5mm of mult_e3] (box_e3) {};
        \node[left=1mm of box_e3, minimum size = 0 mm, inner sep = 0 mm, outer sep = 0 mm] (e1) {$e_3$};
        \node[box, below=5mm of mult_e3] (normal_y3) {$\mathcal{N}$};
        \node[blackbox, left=4mm of normal_y3] (box_tau3) {};
        \node[left=1mm of box_tau3, minimum size = 0 mm, inner sep = 0 mm, outer sep = 0 mm] (tau3) {$\tau_3$};
        \node[blackbox, below=5mm of normal_y3] (box_yt3) {};
        \node[below=1.5mm of box_yt3, minimum size = 0 mm, inner sep = 0 mm, outer sep = 0 mm] (yt3) {$y_{i,3}$};

        % message circles (D=3 partial-obs example: y_{i,1} observed, y_{i,2} y_{i,3} unobserved)
        \node[right=3mm of normal_y1] (anc_1) {};
        \node[draw, circle, minimum size=7mm, below=1.5mm of anc_1] (mcircle1) {$1$};
        \node[left=-0.5mm of mcircle1] {$\uparrow$};
        \node[draw, circle, minimum size=7mm, above=1.25mm of anc_1] (mcircle2) {$2$};
        \node[left=-0.5mm of mcircle2] {$\uparrow$};
        \node[draw, circle, minimum size=7mm, above=11.5mm of anc_1] (mcircle3) {$3$};
        \node[left=-0.5mm of mcircle3] {$\uparrow$};
        \node[draw, circle, minimum size=7mm, above=21.5mm of anc_1] (mcircle4) {$4$};
        \node[left=-0.5mm of mcircle4] {$\uparrow$};

        \node[right=3mm of normal_y2] (anc_2) {};
        \node[left=5mm of eq_my2] (anc_5) {};
        \node[draw, dashed, circle, minimum size=7mm, below=1.5mm of anc_2] (mcircle8) {$8$};
        \node[left=-0.5mm of mcircle8] {$\uparrow$};
        \node[draw, dashed, circle, minimum size=7mm, above=1.0mm of anc_2] (mcircle7) {$7$};
        \node[left=-0.5mm of mcircle7] {$\uparrow$};
        \node[draw, dashed, circle, minimum size=7mm, above=11.5mm of anc_2] (mcircle6) {$6$};
        \node[left=-0.5mm of mcircle6] {$\uparrow$};
        \node[draw, dashed, circle, minimum size=7mm, above=2.5mm of anc_5] (mcircle5) {$5$};
        \node[below=0mm of mcircle5] {$\leftarrow$};

        \node[right=3mm of normal_y3] (anc_3) {};
        \node[draw, dashed, circle, minimum size=7mm, below=1.4mm of anc_3] (mcircle11) {$11$};
        \node[left=-0.5mm of mcircle11] {$\uparrow$};
        \node[draw, dashed, circle, minimum size=7mm, above=0.5mm of anc_3] (mcircle10) {$10$};
        \node[left=-0.5mm of mcircle10] {$\uparrow$};
        \node[draw, dashed, circle, minimum size=7mm, above=11.5mm of anc_3] (mcircle9) {$9$};
        \node[left=-0.5mm of mcircle9] {$\uparrow$};
        
        % anchors for big box
        \node[left=2mm of at] (anc_l) {};
        \node[above=2mm of anc_l] (anc_lt) {};
        \node[right=2mm of yt3] (anc_r) {};
        \node[below=2mm of anc_r] (anc_rb) {};

    % edges
        % lines in and out
        \draw (anc_prev) -- (x)node[pos=0, below=1mm] {$z_{i-1}$};
        \draw (anc_hprev) -- (eq_h) {};
        \draw (eq_x) -- (anc_next)node[pos=.1, below=1mm] {$z_i$};
        \draw (eq_h) -- (anc_hnext) {};
        \draw (normal_x0) -- (anc_prev) {};
        \draw (box_h0) -- (anc_hprev) {};

        % top part
        \draw (box_at) -- (x) {};
        \draw (box_qt) -- (normal_x) {};
        \draw (x) -- (normal_x) node[pos=0, above=0mm] {};
        \draw (normal_x) -- (eq_x) node[pos=0, above=0mm] {};
        \draw (eq_x) -- (mult_my1) {};
        \draw (eq_h) |- (mult_my1) {};

        % observation 1
        \draw (mult_my1) -- (eq_my1) {};
        \draw (eq_my1) -- (mult_e1) {};
        \draw (box_e1) -- (mult_e1) {};
        \draw (mult_e1) -- (normal_y1) {};
        \draw (box_tau1) -- (normal_y1) {};
        \draw (normal_y1) -- (box_yt1) {};

        % observation 2
        \draw (eq_my1) -- (eq_my2) {};
        \draw (eq_my2) -- (mult_e2) {};
        \draw (box_e2) -- (mult_e2) {};
        \draw (mult_e2) -- (normal_y2) {};
        \draw (box_tau2) -- (normal_y2) {};
        \draw (normal_y2) -- (box_yt2) {};

        % observation 3
        \draw (eq_my2) -- (eq_my3) {};
        \draw (eq_my3) -- (mult_e3) {};
        \draw (box_e3) -- (mult_e3) {};
        \draw (mult_e3) -- (normal_y3) {};
        \draw (box_tau3) -- (normal_y3) {};
        \draw (normal_y3) -- (box_yt3) {};

        % big box
        \node[draw, dashed, fit=(anc_lt)(anc_rb)(mcircle11), inner ysep=0pt, inner xsep=4pt] {};
        
\end{tikzpicture}}
  \caption{Example Forney-style factor graph of SS-LMC, for $D=3$. The dashed box indicates a plate repeated over all points in the chain. Messages are carried upwards and joined into a total observation message $4$, which updates the chain. If $y_{i,2}$ and $y_{i,3}$ are unobserved, the dashed messages ($5$--$11$) are absent and do not contribute.}
  \label{fig:ffg_example}
  \vspace{-10pt}
\end{figure*}
The full generative model is expressed as a factor graph for message passing (see \cref{fig:ffg_example}):
\begin{align}\label{eq:generative_model}
   & p(z_0) \! =\! \mathcal{N}(z_0 | 0, P)  ,  \ \ p(z_i | z_{i-1}) \! = \! \mathcal{N}(z_i \given A_i z_{i-1}, Q_i)  , \ \ p(y_{id} \given z_i) \! = \! \mathcal{N}(y_{id} \given e_d^{\top} Hz_i, \tau_d^{-1}) ,
\end{align}
where $e_d$ is a canonical basis vector.
Each output at each chain position is an independent scalar observation factor.
The observation factor for output $d$ at chain position $i$ is included in the joint only when $O_{i,d} = 1$, and entries with $O_{i,d} = 0$ contribute no factor so that predictions flow through the graph unaffected.

\subsection{Inference} \label{sec:inference}

Because all noise precisions $\tau_d$ are fixed scalars, the model is fully linear-Gaussian.
Inference is a single forward-backward pass (exact Kalman smoother), automated by RxInfer.jl.
After the candidate chain has been constructed, the per-step inference cost is $\bigO(C(DL^2 + L^3))$ for fixed Mat\'ern state dimension.
The block-diagonal structure of $A_i$ and $Q_i$ keeps the latent dynamics sparse, but the LMC observation matrix $H$ couples the $L$ latent processes, so exact inference maintains dense Gaussian beliefs over the joint $2L$-dimensional state.
The $D$ per-output scalar observation factors contribute $\bigO(DL^2)$ work per chain position, and the dense Gaussian state updates contribute the $\bigO(L^3)$ term.
% This cost is constant throughout the BO episode, since $C$, $D$, and $L$ are all fixed at setup.
With partial observations, the effective $D$ at each chain position may be smaller, reducing the cost further. We detail this modularity next.

% \subsubsection{Partial Observations via Modularity}
% \label{sec:partial_obs}
\paragraph{Partial Observations via Modularity}

In multi-output settings it is common for different outputs to be observed at different inputs.
In sensor networks, for instance, sensors at different locations measure different quantities at different times \citep{liu_eventbased_2015}.
The mask $O$ of \cref{sec:problem} encodes exactly this pattern, with $O_{i,d} = 0$ marking unobserved entries.

The factor graph formulation handles arbitrary masks natively.
Each potential observation $y_{i,d}$ corresponds to an independent scalar factor, and the local update at chain position $i$ only fires the factors for which $O_{i,d} = 1$.
No covariance-matrix restructuring is needed and no code path changes, so the same model and the same inference call produce posteriors regardless of $O$.
The local message count drops when $O$ has zeros, although fixed overheads may dominate at moderate $D$.
\Cref{fig:ffg_example} visualizes this for $O_{i,2} = O_{i,3} = 0$, where the dashed messages ($5$--$11$) are absent and the subgraph drops from $11$ messages to $4$.

This contrasts with dense kernel-matrix GP baselines, where partial observations require explicitly constructing or slicing a variable-size $P \times P$ kernel matrix from the $P$ observed output--point pairs. The representational model is the same, but the dense covariance implementation has cost $\bigO(P^3)$ coupled to the number of observed entries.

\section{Experiments}
\label{sec:evaluation}

We compare the proposed state-space LMC (SS-LMC) against an exact kernel-matrix LMC (KM-LMC), a sparse-variational inducing-point LMC (SVGP-LMC), and a nearest-neighbor LMC (NNGP-LMC) in two studies: an input-dimension sweep on a synthetic sensor network benchmark that probes the chain approximation (\cref{sec:dim_sweep}), and a real-data forecasting task on ETTh1 that probes the cost of partial-observation updating (\cref{sec:ett}).
All methods use identical LMC model structure: the same mixing matrix~$W$, length-scales~$\ell_l$, output scales~$\gamma^2_l$, and fixed diagonal noise.
The comparison isolates the effect of the inference method and candidate-chain approximation under controlled hyperparameters, reported through held-out posterior-quality metrics and wall-clock time.

RxInfer.jl is used for a reactive message passing implementation (RMP)\footnote[1]{Experiment code available at \url{https://github.com/biaslab/PGM_2026_SSMOGP}} \citep{bagaev_rxinfer_2023}. All timings compare CPU implementations in the same pipeline on a MacBook Pro with Apple M1 Pro CPU (8 cores) and 32\,GB RAM.
% Prior work has applied this framework to time-varying autoregressive models \citep{podusenko_aida_2022} and model selection \citep{vanerp_automating_2023}.
% The present work extends this line by using RxInfer.jl's factor graph infrastructure for Bayesian optimization, demonstrating that the surrogate model and predictions can be handled by a single declarative probabilistic program with compositional support for partial observations.
% This declarative specification contrasts with hand-coded Kalman filter implementations and requirs only local modifications to the graph rather than changes to the inference algorithm.
\subsection{Accuracy as a Function of Input Dimensionality}
\label{sec:dim_sweep}

We probe the nearest-neighbor chain approximation directly by sweeping the input dimension $M$ on a synthetic sensor network benchmark. Each of the $M$ input coordinates is one weather station at a fixed location, and the $D{=}3$ outputs are spatially weighted, saturating combinations of the per-station readings, so increasing $M$ enlarges the input geometry that the one-dimensional chain must compress rather than merely duplicating information. We draw a candidate set of size $C{=}2000$, observe a uniformly random $50\%$ as training data and hold out the remaining $50\%$ for evaluation, and fit four methods with identical LMC structure ($L{=}2$, $\ell_l \in \{1.0, 2.0\}$, $\gamma^2_l \in \{2.0, 1.0\}$, fixed diagonal noise): the proposed SS-LMC, an exact kernel-matrix KM-LMC, a sparse-variational SVGP-LMC with $64$ inducing points per latent, and an NNGP-LMC conditioning each position on its $k{=}20$ nearest locations. For each method we report held-out RMSE and mean negative log-likelihood (MNLL) of the held-out observations, alongside the wall-clock runtime. We sweep $M \in \{2, 4, 8, 16, 32\}$ over $100$ seeds and report mean $\pm$ standard deviation.

\cref{app:chain_gap} reports the nearest-neighbor chain stretch $\Delta$, i.e., the consecutive inter-point distance between chain-ordered candidates, along with the difference in predictive performance between SS-LMC and KM-LMC.

\paragraph{Results.}
\Cref{fig:dim_sweep} reports held-out accuracy as $M$ grows, now alongside the exact kernel-matrix posterior, the NNGP-LMC and the SVGP-LMC baseline. At low input dimension the chain preserves the local geometry (mean $\Delta \approx 0.10$ at $M{=}2$) and SS-LMC closely tracks the exact KM-LMC posterior: held-out RMSE is $0.061$ versus $0.058$ at $M{=}2$ and $0.055$ versus $0.046$ at $M{=}4$, with SVGP-LMC in the same range. As $M$ increases the chain stretches (mean $\Delta$ rises to $5.6$ and max $\Delta$ to $11.2$ at $M{=}32$) and every method's error grows, since at fixed $C$ the regression problem itself becomes harder. The two approximations degrade gracefully rather than catastrophically: at $M{=}32$ SS-LMC reaches RMSE $0.531$ and SVGP-LMC $0.495$, against $0.417$ for exact KM-LMC and $0.524$ for NNGP-LMC, and the SS-LMC--KM-LMC RMSE gap stays $\le$ $0.12$ throughout. SS-LMC's MNLL lies above the exact baseline across the range and stays close to SVGP-LMC's ($1.29$ against $1.26$ at $M{=}32$). NNGP-LMC shows the sharpest dependence on input dimension: it is indistinguishable from the exact posterior at $M{=}2$ (RMSE $0.0584$ for both) and close at $M{=}4$, where $k{=}20$ neighboring locations capture nearly all the relevant correlation, but by $M{=}32$ its MNLL is the worst of the four ($1.50$). Cost separates the methods far more sharply than accuracy does: one smoothing sweep takes $0.015$--$0.017$\,s and is flat in $M$, as $\bigO(C(DL^2+L^3))$ predicts, against $1.27$--$1.36$\,s for exact KM-LMC, $0.07$--$0.17$\,s for SVGP-LMC and $0.41$--$0.51$\,s for NNGP-LMC.

\begin{figure*}[tb]
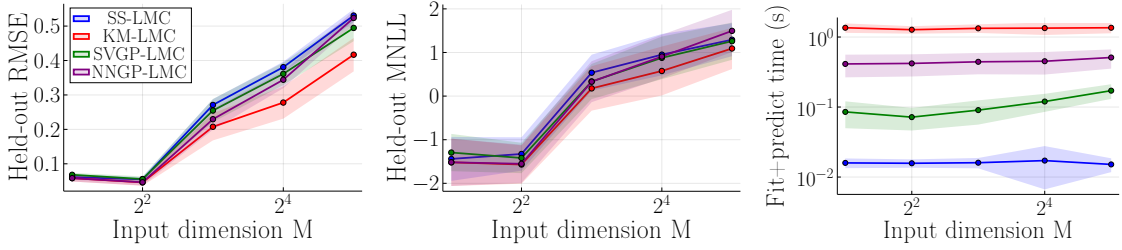

  \centering
  \resizebox{.32\textwidth}{!}{\input{figures/dim_sweep/rmse_vs_d.tikz}}
  \resizebox{.32\textwidth}{!}{\input{figures/dim_sweep/mnll_vs_d.tikz}}
  \resizebox{.32\textwidth}{!}{\input{figures/dim_sweep/time_vs_d.tikz}}
  \caption{Input-dimension sweep on the synthetic sensor network ($C{=}2000$, $D{=}3$, $L{=}2$, $50/50$ train/test, $100$ seeds; mean $\pm$ std).}
  \label{fig:dim_sweep}
  \vspace{-10pt}
\end{figure*}

\subsection{Cost of Partial Observation Updating}
\label{sec:ett}

We forecast the ETTh1 dataset \citep{zhou_informer_2021} in a multidimensional-input regime that exercises the chain approximation on real data. We split the seven channels into an $M{=}3$ input (the useful-load measurements HUFL, MUFL, LUFL) and $D{=}4$ correlated outputs (the useless-load channels HULL, MULL, LULL and oil temperature OT), and order the inputs jointly with the nearest-neighbor chain so that SS-LMC performs $\bigO(N)$ message passing on three-dimensional inputs. We take a contiguous block of $2N$ rows, train on the first $N$ with a training mask $O_{i,d} \sim \mathrm{Bernoulli}(1-p)$ drawn i.i.d.\ across $(i,d)$, and forecast the fully held-out next $N$ rows. We compare four methods with identical LMC structure ($L{=}3$, $\ell_l \in \{0.5, 1, 2\}$, $\gamma^2_l{=}1$, fixed diagonal noise $R{=}0.2$): the proposed SS-LMC by reactive message passing, an exact KM-LMC by covariance restructuring, a sparse-variational SVGP-LMC with $64$ inducing points per latent, and an NNGP-LMC \citep{datta_hierarchical_2016,katzfuss_general_2021} conditioning each position on its $k{=}20$ nearest \emph{locations} ($kD{=}80$ responses, matched to SVGP-LMC's budget). KM-LMC forms the full structured LMC kernel over the observed (input, output) pairs and refactorizes. Because this dense $(DN)$-scale factorization is memory-bound, we run it only up to $N{=}2000$. SS-LMC simply omits the local observation factor of each missing entry, and SVGP-LMC sums per-datum evidence terms over the observed entries. We report held-out forecast MNLL and RMSE and fit+forecast wall-clock time, sweeping the window length $N \in \{500, 1000, 2000, 4000, 8000\}$ at $p{=}0.3$ and the dropout rate $p \in \{0.1, 0.3, 0.5, 0.7\}$ at $N{=}2000$, over $10$ seeds (mean). The complementary $M{=}1$ temporal case, where the natural ordering makes the SS-LMC and KM-LMC posteriors coincide exactly so the comparison is purely computational, is reported in \cref{app:ett_m1}.

\paragraph{Results.}
\Cref{fig:ETT} reports the four methods across window length and dropout. Because the input is now three-dimensional, the chain is an approximation and SS-LMC no longer coincides with KM-LMC, but its forecast accuracy stays close: at $N{=}2000$ the held-out RMSE is $8.17$ for SS-LMC against $7.98$ for exact KM-LMC, $8.10$ for SVGP-LMC and $7.98$ for NNGP-LMC, and at $N{=}8000$ (beyond KM-LMC's reach) SS-LMC reaches RMSE $5.15$ against $4.95$ for SVGP-LMC and $5.19$ for NNGP-LMC, with SVGP-LMC holding a small MNLL advantage throughout and SS-LMC ahead of KM-LMC and NNGP-LMC on MNLL. The cost, however, separates the methods sharply. KM-LMC's covariance restructuring scales cubically ($0.62 \to 2.53 \to 13.2$\,s for $N{=}500 \to 2000$) and is infeasible beyond $N{=}2000$, whereas SS-LMC's single forward--backward sweep grows linearly and reaches $0.16$\,s at $N{=}8000$, $7.6\times$ below the inducing-point SVGP-LMC ($1.25$\,s) and $38\times$ below NNGP-LMC ($6.16$\,s).
The dropout sweep shows the complementary effect: SS-LMC's cost is essentially invariant to the missingness rate (${\sim}0.036$\,s across $p$), since missing entries merely omit local factors, while KM-LMC's cost falls with sparser data ($18.5 \to 5.7$\,s as $p{:}\,0.1 \to 0.7$). SVGP-LMC stays near $0.3$\,s and NNGP-LMC from $1.33$ to $0.61$\,s. \Cref{app:pareto} recasts both views as an accuracy--cost frontier at fixed window length.
Handling partial observations therefore requires no covariance-matrix restructuring in SS-LMC, and the chain approximation costs little accuracy relative to an exact, a sparse-variational and a nearest-neighbor LMC baseline.

\begin{figure*}[tb]
  \centering
  \resizebox{.32\textwidth}{!}{% Recommended preamble:
% \usetikzlibrary{arrows.meta}
% \usetikzlibrary{backgrounds}
% \usepgfplotslibrary{patchplots}
% \usepgfplotslibrary{fillbetween}
% \pgfplotsset{%
%     layers/standard/.define layer set={%
%         background,axis background,axis grid,axis ticks,axis lines,axis tick labels,pre main,main,axis descriptions,axis foreground%
%     }{
%         grid style={/pgfplots/on layer=axis grid},%
%         tick style={/pgfplots/on layer=axis ticks},%
%         axis line style={/pgfplots/on layer=axis lines},%
%         label style={/pgfplots/on layer=axis descriptions},%
%         legend style={/pgfplots/on layer=axis descriptions},%
%         title style={/pgfplots/on layer=axis descriptions},%
%         colorbar style={/pgfplots/on layer=axis descriptions},%
%         ticklabel style={/pgfplots/on layer=axis tick labels},%
%         axis background@ style={/pgfplots/on layer=axis background},%
%         3d box foreground style={/pgfplots/on layer=axis foreground},%
%     },
% }

\begin{tikzpicture}[/tikz/background rectangle/.style={fill={rgb,1:red,1.0;green,1.0;blue,1.0}, fill opacity={1.0}, draw opacity={1.0}}, show background rectangle]
\begin{axis}[point meta max={nan}, point meta min={nan}, legend cell align={left}, legend columns={1}, title={}, title style={at={{(0.5,1)}}, anchor={south}, font={{\fontsize{26 pt}{33.800000000000004 pt}\selectfont}}, color={rgb,1:red,0.0;green,0.0;blue,0.0}, draw opacity={1.0}, rotate={0.0}, align={center}}, legend style={color={rgb,1:red,0.0;green,0.0;blue,0.0}, draw opacity={1.0}, line width={1}, solid, fill={rgb,1:red,1.0;green,1.0;blue,1.0}, fill opacity={1.0}, text opacity={1.0}, font={{\fontsize{20 pt}{26.0 pt}\selectfont}}, text={rgb,1:red,0.0;green,0.0;blue,0.0}, cells={anchor={center}}, at={(1.02, 1)}, anchor={north west}}, axis background/.style={fill={rgb,1:red,1.0;green,1.0;blue,1.0}, opacity={1.0}}, anchor={north west}, xshift={8.0mm}, yshift={-1.0mm}, width={133.24mm}, height={89.52mm}, scaled x ticks={false}, xlabel={Window length C (p=0.3)}, x tick style={color={rgb,1:red,0.0;green,0.0;blue,0.0}, opacity={1.0}}, x tick label style={color={rgb,1:red,0.0;green,0.0;blue,0.0}, opacity={1.0}, rotate={0}}, xlabel style={at={(ticklabel cs:0.5)}, anchor=near ticklabel, at={{(ticklabel cs:0.5)}}, anchor={near ticklabel}, font={{\fontsize{28 pt}{36.4 pt}\selectfont}}, color={rgb,1:red,0.0;green,0.0;blue,0.0}, draw opacity={1.0}, rotate={0.0}}, xmajorgrids={true}, xmin={275.0}, xmax={8225.0}, xticklabels={{$2000$,$4000$,$6000$,$8000$}}, xtick={{2000.0,4000.0,6000.0,8000.0}}, xtick align={inside}, xticklabel style={font={{\fontsize{24 pt}{31.200000000000003 pt}\selectfont}}, color={rgb,1:red,0.0;green,0.0;blue,0.0}, draw opacity={1.0}, rotate={0.0}}, x grid style={color={rgb,1:red,0.0;green,0.0;blue,0.0}, draw opacity={0.1}, line width={0.5}, solid}, axis x line*={left}, x axis line style={color={rgb,1:red,0.0;green,0.0;blue,0.0}, draw opacity={1.0}, line width={1}, solid}, scaled y ticks={false}, ylabel={Fit+forecast time (s)}, y tick style={color={rgb,1:red,0.0;green,0.0;blue,0.0}, opacity={1.0}}, y tick label style={color={rgb,1:red,0.0;green,0.0;blue,0.0}, opacity={1.0}, rotate={0}}, ylabel style={at={(ticklabel cs:0.5)}, anchor=near ticklabel, at={{(ticklabel cs:0.5)}}, anchor={near ticklabel}, font={{\fontsize{28 pt}{36.4 pt}\selectfont}}, color={rgb,1:red,0.0;green,0.0;blue,0.0}, draw opacity={1.0}, rotate={0.0}}, ymode={log}, log basis y={10}, ymajorgrids={true}, ymin={0.0059348444375000005}, ymax={18.483549739340003}, yticklabels={{0.01,0.1,1,10}}, ytick={{0.01,0.1,1.0,10.0}}, ytick align={inside}, yticklabel style={font={{\fontsize{24 pt}{31.200000000000003 pt}\selectfont}}, color={rgb,1:red,0.0;green,0.0;blue,0.0}, draw opacity={1.0}, rotate={0.0}}, y grid style={color={rgb,1:red,0.0;green,0.0;blue,0.0}, draw opacity={0.1}, line width={0.5}, solid}, axis y line*={left}, y axis line style={color={rgb,1:red,0.0;green,0.0;blue,0.0}, draw opacity={1.0}, line width={1}, solid}, colorbar={false}]
    \addplot[color={rgb,1:red,0.0;green,0.0;blue,1.0}, name path={1}, draw opacity={1.0}, line width={2}, solid, mark={*}, mark size={3.0 pt}, mark repeat={1}, mark options={color={rgb,1:red,0.0;green,0.0;blue,0.0}, draw opacity={1.0}, fill={rgb,1:red,0.0;green,0.0;blue,1.0}, fill opacity={1.0}, line width={0.75}, rotate={0}, solid}]
        table[row sep={\\}]
        {
            \\
            500.0  0.009495751100000002  \\
            1000.0  0.022885803000000003  \\
            2000.0  0.0360167659  \\
            4000.0  0.0832467479  \\
            8000.0  0.16315844170000002  \\
        }
        ;
    \addplot[color={rgb,1:red,1.0;green,0.0;blue,0.0}, name path={2}, draw opacity={1.0}, line width={2}, solid, mark={*}, mark size={3.0 pt}, mark repeat={1}, mark options={color={rgb,1:red,0.0;green,0.0;blue,0.0}, draw opacity={1.0}, fill={rgb,1:red,1.0;green,0.0;blue,0.0}, fill opacity={1.0}, line width={0.75}, rotate={0}, solid}]
        table[row sep={\\}]
        {
            \\
            500.0  0.6199823077  \\
            1000.0  2.5283786599000004  \\
            2000.0  13.202535528100004  \\
        }
        ;
    \addplot[color={rgb,1:red,0.0;green,0.502;blue,0.0}, name path={3}, draw opacity={1.0}, line width={2}, solid, mark={*}, mark size={3.0 pt}, mark repeat={1}, mark options={color={rgb,1:red,0.0;green,0.0;blue,0.0}, draw opacity={1.0}, fill={rgb,1:red,0.0;green,0.502;blue,0.0}, fill opacity={1.0}, line width={0.75}, rotate={0}, solid}]
        table[row sep={\\}]
        {
            \\
            500.0  0.0759865605  \\
            1000.0  0.1355278819  \\
            2000.0  0.3015992828  \\
            4000.0  0.5672922373  \\
            8000.0  1.2482629603000002  \\
        }
        ;
    \addplot[color={rgb,1:red,0.502;green,0.0;blue,0.502}, name path={4}, draw opacity={1.0}, line width={2}, solid, mark={*}, mark size={3.0 pt}, mark repeat={1}, mark options={color={rgb,1:red,0.0;green,0.0;blue,0.0}, draw opacity={1.0}, fill={rgb,1:red,0.502;green,0.0;blue,0.502}, fill opacity={1.0}, line width={0.75}, rotate={0}, solid}]
        table[row sep={\\}]
        {
            \\
            500.0  0.22503575329999997  \\
            1000.0  0.6812491268  \\
            2000.0  0.8605454709  \\
            4000.0  2.3808842453  \\
            8000.0  6.1557580854000005  \\
        }
        ;
\end{axis}
\end{tikzpicture}}
  \resizebox{.32\textwidth}{!}{% Recommended preamble:
% \usetikzlibrary{arrows.meta}
% \usetikzlibrary{backgrounds}
% \usepgfplotslibrary{patchplots}
% \usepgfplotslibrary{fillbetween}
% \pgfplotsset{%
%     layers/standard/.define layer set={%
%         background,axis background,axis grid,axis ticks,axis lines,axis tick labels,pre main,main,axis descriptions,axis foreground%
%     }{
%         grid style={/pgfplots/on layer=axis grid},%
%         tick style={/pgfplots/on layer=axis ticks},%
%         axis line style={/pgfplots/on layer=axis lines},%
%         label style={/pgfplots/on layer=axis descriptions},%
%         legend style={/pgfplots/on layer=axis descriptions},%
%         title style={/pgfplots/on layer=axis descriptions},%
%         colorbar style={/pgfplots/on layer=axis descriptions},%
%         ticklabel style={/pgfplots/on layer=axis tick labels},%
%         axis background@ style={/pgfplots/on layer=axis background},%
%         3d box foreground style={/pgfplots/on layer=axis foreground},%
%     },
% }

\begin{tikzpicture}[/tikz/background rectangle/.style={fill={rgb,1:red,1.0;green,1.0;blue,1.0}, fill opacity={1.0}, draw opacity={1.0}}, show background rectangle]
\begin{axis}[point meta max={nan}, point meta min={nan}, legend cell align={left}, legend columns={1}, title={}, title style={at={{(0.5,1)}}, anchor={south}, font={{\fontsize{26 pt}{33.800000000000004 pt}\selectfont}}, color={rgb,1:red,0.0;green,0.0;blue,0.0}, draw opacity={1.0}, rotate={0.0}, align={center}}, legend style={color={rgb,1:red,0.0;green,0.0;blue,0.0}, draw opacity={1.0}, line width={1}, solid, fill={rgb,1:red,1.0;green,1.0;blue,1.0}, fill opacity={1.0}, text opacity={1.0}, font={{\fontsize{20 pt}{26.0 pt}\selectfont}}, text={rgb,1:red,0.0;green,0.0;blue,0.0}, cells={anchor={center}}, at={(0.98, 0.98)}, anchor={north east}}, axis background/.style={fill={rgb,1:red,1.0;green,1.0;blue,1.0}, opacity={1.0}}, anchor={north west}, xshift={8.0mm}, yshift={-1.0mm}, width={133.24mm}, height={89.52mm}, scaled x ticks={false}, xlabel={Window length C (p=0.3)}, x tick style={color={rgb,1:red,0.0;green,0.0;blue,0.0}, opacity={1.0}}, x tick label style={color={rgb,1:red,0.0;green,0.0;blue,0.0}, opacity={1.0}, rotate={0}}, xlabel style={at={(ticklabel cs:0.5)}, anchor=near ticklabel, at={{(ticklabel cs:0.5)}}, anchor={near ticklabel}, font={{\fontsize{28 pt}{36.4 pt}\selectfont}}, color={rgb,1:red,0.0;green,0.0;blue,0.0}, draw opacity={1.0}, rotate={0.0}}, xmajorgrids={true}, xmin={275.0}, xmax={8225.0}, xticklabels={{$2000$,$4000$,$6000$,$8000$}}, xtick={{2000.0,4000.0,6000.0,8000.0}}, xtick align={inside}, xticklabel style={font={{\fontsize{24 pt}{31.200000000000003 pt}\selectfont}}, color={rgb,1:red,0.0;green,0.0;blue,0.0}, draw opacity={1.0}, rotate={0.0}}, x grid style={color={rgb,1:red,0.0;green,0.0;blue,0.0}, draw opacity={0.1}, line width={0.5}, solid}, axis x line*={left}, x axis line style={color={rgb,1:red,0.0;green,0.0;blue,0.0}, draw opacity={1.0}, line width={1}, solid}, scaled y ticks={false}, ylabel={Forecast MNLL}, y tick style={color={rgb,1:red,0.0;green,0.0;blue,0.0}, opacity={1.0}}, y tick label style={color={rgb,1:red,0.0;green,0.0;blue,0.0}, opacity={1.0}, rotate={0}}, ylabel style={at={(ticklabel cs:0.5)}, anchor=near ticklabel, at={{(ticklabel cs:0.5)}}, anchor={near ticklabel}, font={{\fontsize{28 pt}{36.4 pt}\selectfont}}, color={rgb,1:red,0.0;green,0.0;blue,0.0}, draw opacity={1.0}, rotate={0.0}}, ymajorgrids={true}, ymin={-2.5848866027035533}, ymax={23.34194682164489}, yticklabels={{$0$,$5$,$10$,$15$,$20$}}, ytick={{0.0,5.0,10.0,15.0,20.0}}, ytick align={inside}, yticklabel style={font={{\fontsize{24 pt}{31.200000000000003 pt}\selectfont}}, color={rgb,1:red,0.0;green,0.0;blue,0.0}, draw opacity={1.0}, rotate={0.0}}, y grid style={color={rgb,1:red,0.0;green,0.0;blue,0.0}, draw opacity={0.1}, line width={0.5}, solid}, axis y line*={left}, y axis line style={color={rgb,1:red,0.0;green,0.0;blue,0.0}, draw opacity={1.0}, line width={1}, solid}, colorbar={false}]
    \addplot[color={rgb,1:red,0.0;green,0.0;blue,1.0}, name path={5}, draw opacity={1.0}, line width={2}, solid, mark={*}, mark size={3.0 pt}, mark repeat={1}, mark options={color={rgb,1:red,0.0;green,0.0;blue,0.0}, draw opacity={1.0}, fill={rgb,1:red,0.0;green,0.0;blue,1.0}, fill opacity={1.0}, line width={0.75}, rotate={0}, solid}]
        table[row sep={\\}]
        {
            \\
            500.0  4.333310848697207  \\
            1000.0  3.754669444766513  \\
            2000.0  11.27769661559979  \\
            4000.0  7.1689132078253595  \\
            8000.0  3.275974060557284  \\
        }
        ;
    \addlegendentry {SS-LMC (ours)}
    \addplot[color={rgb,1:red,1.0;green,0.0;blue,0.0}, name path={6}, draw opacity={1.0}, line width={2}, solid, mark={*}, mark size={3.0 pt}, mark repeat={1}, mark options={color={rgb,1:red,0.0;green,0.0;blue,0.0}, draw opacity={1.0}, fill={rgb,1:red,1.0;green,0.0;blue,0.0}, fill opacity={1.0}, line width={0.75}, rotate={0}, solid}]
        table[row sep={\\}]
        {
            \\
            500.0  4.906096374115297  \\
            1000.0  4.1780049270785105  \\
            2000.0  13.288684881591411  \\
        }
        ;
    \addlegendentry {KM-LMC}
    \addplot[color={rgb,1:red,0.0;green,0.502;blue,0.0}, name path={7}, draw opacity={1.0}, line width={2}, solid, mark={*}, mark size={3.0 pt}, mark repeat={1}, mark options={color={rgb,1:red,0.0;green,0.0;blue,0.0}, draw opacity={1.0}, fill={rgb,1:red,0.0;green,0.502;blue,0.0}, fill opacity={1.0}, line width={0.75}, rotate={0}, solid}]
        table[row sep={\\}]
        {
            \\
            500.0  3.9336522129307285  \\
            1000.0  3.3929800893077733  \\
            2000.0  9.214455350434468  \\
            4000.0  5.368011404082851  \\
            8000.0  2.706303892061435  \\
        }
        ;
    \addlegendentry {SVGP-LMC}
    \addplot[color={rgb,1:red,0.502;green,0.0;blue,0.502}, name path={8}, draw opacity={1.0}, line width={2}, solid, mark={*}, mark size={3.0 pt}, mark repeat={1}, mark options={color={rgb,1:red,0.0;green,0.0;blue,0.0}, draw opacity={1.0}, fill={rgb,1:red,0.502;green,0.0;blue,0.502}, fill opacity={1.0}, line width={0.75}, rotate={0}, solid}]
        table[row sep={\\}]
        {
            \\
            500.0  4.905741659367152  \\
            1000.0  4.1732580038645395  \\
            2000.0  13.2219783714655  \\
            4000.0  8.504451675290465  \\
            8000.0  3.5664627977145975  \\
        }
        ;
    \addlegendentry {NNGP-LMC}
\end{axis}
\end{tikzpicture}}
  \resizebox{.32\textwidth}{!}{% Recommended preamble:
% \usetikzlibrary{arrows.meta}
% \usetikzlibrary{backgrounds}
% \usepgfplotslibrary{patchplots}
% \usepgfplotslibrary{fillbetween}
% \pgfplotsset{%
%     layers/standard/.define layer set={%
%         background,axis background,axis grid,axis ticks,axis lines,axis tick labels,pre main,main,axis descriptions,axis foreground%
%     }{
%         grid style={/pgfplots/on layer=axis grid},%
%         tick style={/pgfplots/on layer=axis ticks},%
%         axis line style={/pgfplots/on layer=axis lines},%
%         label style={/pgfplots/on layer=axis descriptions},%
%         legend style={/pgfplots/on layer=axis descriptions},%
%         title style={/pgfplots/on layer=axis descriptions},%
%         colorbar style={/pgfplots/on layer=axis descriptions},%
%         ticklabel style={/pgfplots/on layer=axis tick labels},%
%         axis background@ style={/pgfplots/on layer=axis background},%
%         3d box foreground style={/pgfplots/on layer=axis foreground},%
%     },
% }

\begin{tikzpicture}[/tikz/background rectangle/.style={fill={rgb,1:red,1.0;green,1.0;blue,1.0}, fill opacity={1.0}, draw opacity={1.0}}, show background rectangle]
\begin{axis}[point meta max={nan}, point meta min={nan}, legend cell align={left}, legend columns={1}, title={}, title style={at={{(0.5,1)}}, anchor={south}, font={{\fontsize{26 pt}{33.800000000000004 pt}\selectfont}}, color={rgb,1:red,0.0;green,0.0;blue,0.0}, draw opacity={1.0}, rotate={0.0}, align={center}}, legend style={color={rgb,1:red,0.0;green,0.0;blue,0.0}, draw opacity={1.0}, line width={1}, solid, fill={rgb,1:red,1.0;green,1.0;blue,1.0}, fill opacity={1.0}, text opacity={1.0}, font={{\fontsize{20 pt}{26.0 pt}\selectfont}}, text={rgb,1:red,0.0;green,0.0;blue,0.0}, cells={anchor={center}}, at={(1.02, 1)}, anchor={north west}}, axis background/.style={fill={rgb,1:red,1.0;green,1.0;blue,1.0}, opacity={1.0}}, anchor={north west}, xshift={8.0mm}, yshift={-1.0mm}, width={133.24mm}, height={89.52mm}, scaled x ticks={false}, xlabel={Dropout p (C=2000)}, x tick style={color={rgb,1:red,0.0;green,0.0;blue,0.0}, opacity={1.0}}, x tick label style={color={rgb,1:red,0.0;green,0.0;blue,0.0}, opacity={1.0}, rotate={0}}, xlabel style={at={(ticklabel cs:0.5)}, anchor=near ticklabel, at={{(ticklabel cs:0.5)}}, anchor={near ticklabel}, font={{\fontsize{28 pt}{36.4 pt}\selectfont}}, color={rgb,1:red,0.0;green,0.0;blue,0.0}, draw opacity={1.0}, rotate={0.0}}, xmajorgrids={true}, xmin={0.08199999999999996}, xmax={0.718}, xticklabels={{$0.1$,$0.2$,$0.3$,$0.4$,$0.5$,$0.6$,$0.7$}}, xtick={{0.1,0.2,0.30000000000000004,0.4,0.5,0.6000000000000001,0.7000000000000001}}, xtick align={inside}, xticklabel style={font={{\fontsize{24 pt}{31.200000000000003 pt}\selectfont}}, color={rgb,1:red,0.0;green,0.0;blue,0.0}, draw opacity={1.0}, rotate={0.0}}, x grid style={color={rgb,1:red,0.0;green,0.0;blue,0.0}, draw opacity={0.1}, line width={0.5}, solid}, axis x line*={left}, x axis line style={color={rgb,1:red,0.0;green,0.0;blue,0.0}, draw opacity={1.0}, line width={1}, solid}, scaled y ticks={false}, ylabel={Fit+forecast time (s)}, y tick style={color={rgb,1:red,0.0;green,0.0;blue,0.0}, opacity={1.0}}, y tick label style={color={rgb,1:red,0.0;green,0.0;blue,0.0}, opacity={1.0}, rotate={0}}, ylabel style={at={(ticklabel cs:0.5)}, anchor=near ticklabel, at={{(ticklabel cs:0.5)}}, anchor={near ticklabel}, font={{\fontsize{28 pt}{36.4 pt}\selectfont}}, color={rgb,1:red,0.0;green,0.0;blue,0.0}, draw opacity={1.0}, rotate={0.0}}, ymode={log}, log basis y={10}, ymajorgrids={true}, ymin={0.021478289124999994}, ymax={25.929812578319996}, yticklabels={{0.1,1,10}}, ytick={{0.1,1.0,10.0}}, ytick align={inside}, yticklabel style={font={{\fontsize{24 pt}{31.200000000000003 pt}\selectfont}}, color={rgb,1:red,0.0;green,0.0;blue,0.0}, draw opacity={1.0}, rotate={0.0}}, y grid style={color={rgb,1:red,0.0;green,0.0;blue,0.0}, draw opacity={0.1}, line width={0.5}, solid}, axis y line*={left}, y axis line style={color={rgb,1:red,0.0;green,0.0;blue,0.0}, draw opacity={1.0}, line width={1}, solid}, colorbar={false}]
    \addplot[color={rgb,1:red,0.0;green,0.0;blue,1.0}, name path={13}, draw opacity={1.0}, line width={2}, solid, mark={*}, mark size={3.0 pt}, mark repeat={1}, mark options={color={rgb,1:red,0.0;green,0.0;blue,0.0}, draw opacity={1.0}, fill={rgb,1:red,0.0;green,0.0;blue,1.0}, fill opacity={1.0}, line width={0.75}, rotate={0}, solid}]
        table[row sep={\\}]
        {
            \\
            0.1  0.0355883626  \\
            0.3  0.0357941503  \\
            0.5  0.0438301901  \\
            0.7  0.034365262599999995  \\
        }
        ;
    \addplot[color={rgb,1:red,1.0;green,0.0;blue,0.0}, name path={14}, draw opacity={1.0}, line width={2}, solid, mark={*}, mark size={3.0 pt}, mark repeat={1}, mark options={color={rgb,1:red,0.0;green,0.0;blue,0.0}, draw opacity={1.0}, fill={rgb,1:red,1.0;green,0.0;blue,0.0}, fill opacity={1.0}, line width={0.75}, rotate={0}, solid}]
        table[row sep={\\}]
        {
            \\
            0.1  18.5212946988  \\
            0.3  13.5688483843  \\
            0.5  8.793470715000002  \\
            0.7  5.658021343  \\
        }
        ;
    \addplot[color={rgb,1:red,0.0;green,0.502;blue,0.0}, name path={15}, draw opacity={1.0}, line width={2}, solid, mark={*}, mark size={3.0 pt}, mark repeat={1}, mark options={color={rgb,1:red,0.0;green,0.0;blue,0.0}, draw opacity={1.0}, fill={rgb,1:red,0.0;green,0.502;blue,0.0}, fill opacity={1.0}, line width={0.75}, rotate={0}, solid}]
        table[row sep={\\}]
        {
            \\
            0.1  0.36188260629999996  \\
            0.3  0.31280264290000004  \\
            0.5  0.21963262339999998  \\
            0.7  0.21023393500000004  \\
        }
        ;
    \addplot[color={rgb,1:red,0.502;green,0.0;blue,0.502}, name path={16}, draw opacity={1.0}, line width={2}, solid, mark={*}, mark size={3.0 pt}, mark repeat={1}, mark options={color={rgb,1:red,0.0;green,0.0;blue,0.0}, draw opacity={1.0}, fill={rgb,1:red,0.502;green,0.0;blue,0.502}, fill opacity={1.0}, line width={0.75}, rotate={0}, solid}]
        table[row sep={\\}]
        {
            \\
            0.1  1.3320356123  \\
            0.3  0.8948118230000001  \\
            0.5  0.6685880291999998  \\
            0.7  0.6112387879  \\
        }
        ;
\end{axis}
\end{tikzpicture}}
  \caption{ETTh1 forecasting at $M{=}3$, $D{=}4$, $L{=}3$, $10$ seeds. KM-LMC runs only to $N{=}2000$. Timings exclude the one-off chain construction (cf.\ \cref{sec:discussion}).}
  \label{fig:ETT}
  \vspace{-10pt}
\end{figure*}
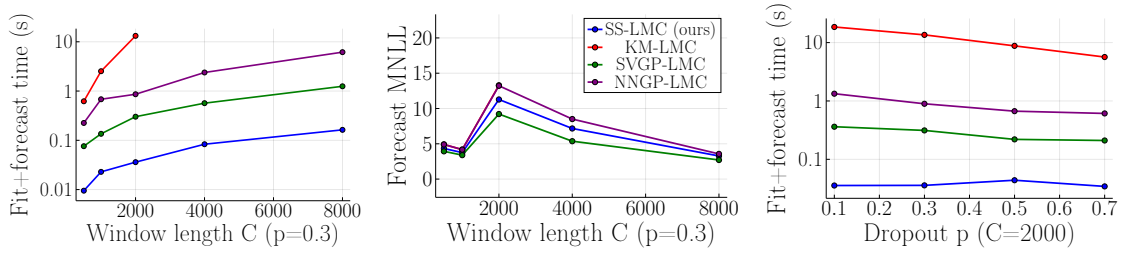

\section{Discussion}
\label{sec:discussion}
\begin{table*}[t]
  \centering
  \caption{Per-step computational cost of the baselines and the proposed SS-LMC. $m$: SVGP inducing points per latent. $k$: NNGP conditioning locations, each carrying $D$ outputs. $P = \sum_{i,d} O_{i,d}$: observed output--input pairs.}
  \label{tab:complexity}
  \resizebox{\textwidth}{!}{%
    \begin{tabular}{|l | c | c | c | c | c |}
      \toprule
                           & Naive MOGP       & KM-LMC                          & SVGP-LMC               & NNGP-LMC                        & SS-LMC (ours)                             \\
      \midrule
      Per-step inference   & $\bigO((DN)^3)$  & $\bigO(LN^3 + N D L^2)$         & $\bigO(L(Nm^2 + m^3))$ & $\bigO(N(kD)^3)$                & $\bigO(C(DL^2 + L^3))$                    \\
      Chain construction   & ---              & ---                             & ---                    & $\bigO(C^2)$ brute force        & $\bigO(C^2)$ brute force                  \\
      Partial observations & $\bigO((DN)^3)$  & $\bigO(P^3)$                    & $\bigO(L(Pm^2 + m^3))$ & $\bigO(N(kD)^3)$                & $\bigO(C(DL^2 + L^3))$                    \\
      Total over $N$ steps & $\bigO(D^3 N^4)$ & $\bigO(LN^4 \! + \! N^2 D L^2)$ & $\bigO(L N m^2)$       & $\bigO(C^2 \! + \! N^2 (kD)^3)$ & $\bigO(C^2 \! + \! NC(DL^2 \! + \! L^3))$ \\
      \bottomrule
    \end{tabular}%
  }
  \vspace{-5pt}
\end{table*}

SS-LMC has $\bigO(C(DL^2+L^3))$ per-step inference after chain construction (\cref{tab:complexity}). NNGP-LMC \citep{datta_hierarchical_2016,katzfuss_general_2021} shares the nearest-neighbor ingredient but conditions each prediction on $k$ neighbor \emph{locations} instead of propagating one state along a chain, and pays the same $\bigO(C^2)$ brute-force construction \citep[\S S1]{finley_efficient_2019}.
Chain construction is preprocessing that the kernel-matrix and inducing-point baselines avoid, but it depends only on the candidate inputs, and not on $Y$, the mask $O$, or the hyperparameters. So it is built once and reused by every later inference call. Writing $c_\pi$ for its cost, $R$ reuses favor SS-LMC once $R > c_\pi / (t_{\mathrm{base}} - t_{\mathrm{SS}})$. Since the one-off chain cost is well below the per-fit gap to either baseline, the construction is recovered within a single fit, and repeated queries, hyperparameter-learning iterations, and streaming updates amortize it further.
The nearest-neighbor chain replaces Euclidean by path distances along a one-dimensional Markov chain, and the input-dimension sweep (\cref{fig:dim_sweep}) shows the resulting trade-off: fidelity is high at low $M$ and degrades gradually as the chain stretches. The method is therefore strongest for low-input-dimensional candidate sets.

\citet{nguyen_factor_2025} also decomposes a Gaussian process into a factor graph formulation, but utilizes variational approximations instead of the state-space formulation. It is a factor graph treatment of the sparse variational Gaussian process, and has the same computational and memory complexity at similar accuracy in regression and classification tasks.

%The same factor-graph compositionality that absorbs missing observations applies to other local extensions --- online noise learning via conjugate priors, input-dependent noise, non-stationary kernels --- by modifying local factors only; we leave such extensions to future work.

\paragraph{Hyperparameter Learning.}
Learning the hyperparameters $(W, \ell_l, \gamma^2_l)$, which are fixed throughout this work, interacts asymmetrically with the graph. The mixing matrix enters only through $H$ and so touches the $C$ observation blocks alone, whereas $\ell_l$ and $\gamma^2_l$ enter through $A_i$ and $Q_i$ and re-parameterize every transition factor: unlike the missingness mask, a hyperparameter update is not a local graph edit. The ordering is spared, since $\pi$ and the $\Delta_i$ depend only on the input geometry and are computed once and reused across iterations. Expectation maximization is then natural, alternating one $\bigO(C(DL^2+L^3))$ smoothing sweep with automatic-differentiation gradient steps on the hyperparameters, taking the free energy returned by that same sweep as the objective, with banded linear algebra for the gradients \citep{durrande_banded_2019}; particle MCMC \citep{svensson_computationally_2016} and recursive Bayesian autoregression \citep{kouw_bayesian_2025} are alternatives. Learned length-scales are moreover fit to chain rather than Euclidean distances, so they absorb part of the chain stretch and inflate $\alpha_l$, and with it the bound of \cref{thm:lmc_posterior_perturbation}.

\paragraph{Limitations.}
Firstly, the greedy nearest-neighbor chain is path-dependent. However, the effect is, so far, empirically negligible: held-out RMSE varies by ${<}0.1\%$ across $20$ starting points (within the $\mathcal{O}(\log C)$ bound of \cref{prop:chain_start}), and even adversarially diverse starts leave it unchanged (\cref{app:chain_start}). Secondly, the ETT study (\cref{sec:ett}) exercises an $M{=}3$ input on real data, but a systematic evaluation at higher input dimension on real data is left to future work. Thirdly, the bounds of \cref{lem:chain_kernel_bound,thm:lmc_posterior_perturbation} are illustrative from a theoretical perspective but do not furnish error estimates one would use in practice. Lastly, the method is currently limited to Mat\'ern-class covariances that admit state-space representations.

\section{Conclusion}
\label{sec:conclusion}

We presented low-input-dimensional multi-output GP inference as message passing on a Forney-style factor graph, combining state-space GPs, LMC, and reactive message passing.
On a synthetic sensor-network benchmark it tracks the exact kernel-matrix posterior at low input dimension and degrades gracefully as the chain stretches with input dimension.
On real ETTh1 forecasting under random per-(row, output) dropout with a three-dimensional input, SS-LMC matches both baselines in forecast accuracy to within a small margin, while its inference cost grows only linearly in the window length and is invariant to the dropout rate.
The factor graph handles partial observations without covariance-matrix restructuring.

\section*{Acknowledgments}
This publication is part of the ROBUST project with project number KICH3.LTP.20.006, which is (partly) financed by the Dutch Research Council (NWO), GN Hearing, and the Dutch Ministry of Economic Affairs and Climate Policy (EZK) under the program LTP KIC 2020-2023. This project is also partly financed by Holland High
Tech with PPS funding for the AUTO-AR project RVO TKI2112P09.

% Bibliography
% PMLR uses author-year citations by default
% Use \citet{} for textual citations: "Smith and Jones (2023) showed..."
% Use \citep{} for parenthetical citations: "as shown previously (Smith and Jones, 2023)"

% You can use BibTeX (recommended)
% For BibTeX, use: 
\bibliography{references}

\appendix
\crefalias{section}{appendix}

\section{Chain Stretch and the Gap to the Exact Posterior} \label{app:chain_gap}

\Cref{fig:chain_gap} reports the geometry the chain has to compress and what it costs against the exact posterior. The stretch $\Delta$ grows steadily with $M$ (mean $0.10 \to 5.6$, max $4.5 \to 11.2$ for $M{:}\,2 \to 32$), and the accuracy gap follows it. The difference between SS-LMC and KM-LMC within each seed rises monotonically in RMSE, from $0.003 \pm 0.002$ at $M{=}2$ to $0.114 \pm 0.035$ at $M{=}32$. The MNLL gap rises likewise, from $0.08 \pm 0.08$ to a maximum of $0.37 \pm 0.29$ at $M{=}16$, then falls back to $0.20 \pm 0.23$ at $M{=}32$, where the regression problem is hard enough at fixed $C$ that the exact posterior loses part of its advantage too.

\begin{figure*}[htb]
  \centering
  \resizebox{.32\textwidth}{!}{% Recommended preamble:
% \usetikzlibrary{arrows.meta}
% \usetikzlibrary{backgrounds}
% \usepgfplotslibrary{patchplots}
% \usepgfplotslibrary{fillbetween}
% \pgfplotsset{%
%     layers/standard/.define layer set={%
%         background,axis background,axis grid,axis ticks,axis lines,axis tick labels,pre main,main,axis descriptions,axis foreground%
%     }{
%         grid style={/pgfplots/on layer=axis grid},%
%         tick style={/pgfplots/on layer=axis ticks},%
%         axis line style={/pgfplots/on layer=axis lines},%
%         label style={/pgfplots/on layer=axis descriptions},%
%         legend style={/pgfplots/on layer=axis descriptions},%
%         title style={/pgfplots/on layer=axis descriptions},%
%         colorbar style={/pgfplots/on layer=axis descriptions},%
%         ticklabel style={/pgfplots/on layer=axis tick labels},%
%         axis background@ style={/pgfplots/on layer=axis background},%
%         3d box foreground style={/pgfplots/on layer=axis foreground},%
%     },
% }

\begin{tikzpicture}[/tikz/background rectangle/.style={fill={rgb,1:red,1.0;green,1.0;blue,1.0}, fill opacity={1.0}, draw opacity={1.0}}, show background rectangle]
\begin{axis}[point meta max={nan}, point meta min={nan}, legend cell align={left}, legend columns={1}, title={}, title style={at={{(0.5,1)}}, anchor={south}, font={{\fontsize{26 pt}{33.800000000000004 pt}\selectfont}}, color={rgb,1:red,0.0;green,0.0;blue,0.0}, draw opacity={1.0}, rotate={0.0}, align={center}}, legend style={color={rgb,1:red,0.0;green,0.0;blue,0.0}, draw opacity={1.0}, line width={1}, solid, fill={rgb,1:red,1.0;green,1.0;blue,1.0}, fill opacity={1.0}, text opacity={1.0}, font={{\fontsize{20 pt}{26.0 pt}\selectfont}}, text={rgb,1:red,0.0;green,0.0;blue,0.0}, cells={anchor={center}}, at={(0.02, 0.98)}, anchor={north west}}, axis background/.style={fill={rgb,1:red,1.0;green,1.0;blue,1.0}, opacity={1.0}}, anchor={north west}, xshift={8.0mm}, yshift={-1.0mm}, width={133.24mm}, height={89.52mm}, scaled x ticks={false}, xlabel={Input dimension M}, x tick style={color={rgb,1:red,0.0;green,0.0;blue,0.0}, opacity={1.0}}, x tick label style={color={rgb,1:red,0.0;green,0.0;blue,0.0}, opacity={1.0}, rotate={0}}, xlabel style={at={(ticklabel cs:0.5)}, anchor=near ticklabel, at={{(ticklabel cs:0.5)}}, anchor={near ticklabel}, font={{\fontsize{28 pt}{36.4 pt}\selectfont}}, color={rgb,1:red,0.0;green,0.0;blue,0.0}, draw opacity={1.0}, rotate={0.0}}, xmode={log}, log basis x={2}, xmajorgrids={true}, xmin={1.8403753012497501}, xmax={34.77551560083386}, xticklabels={{$2^{2}$,$2^{4}$}}, xtick={{4.0,16.0}}, xtick align={inside}, xticklabel style={font={{\fontsize{24 pt}{31.200000000000003 pt}\selectfont}}, color={rgb,1:red,0.0;green,0.0;blue,0.0}, draw opacity={1.0}, rotate={0.0}}, x grid style={color={rgb,1:red,0.0;green,0.0;blue,0.0}, draw opacity={0.1}, line width={0.5}, solid}, axis x line*={left}, x axis line style={color={rgb,1:red,0.0;green,0.0;blue,0.0}, draw opacity={1.0}, line width={1}, solid}, scaled y ticks={false}, ylabel={Chain $\Delta$}, y tick style={color={rgb,1:red,0.0;green,0.0;blue,0.0}, opacity={1.0}}, y tick label style={color={rgb,1:red,0.0;green,0.0;blue,0.0}, opacity={1.0}, rotate={0}}, ylabel style={at={(ticklabel cs:0.5)}, anchor=near ticklabel, at={{(ticklabel cs:0.5)}}, anchor={near ticklabel}, font={{\fontsize{28 pt}{36.4 pt}\selectfont}}, color={rgb,1:red,0.0;green,0.0;blue,0.0}, draw opacity={1.0}, rotate={0.0}}, ymajorgrids={true}, ymin={-0.2618077729397701}, ymax={12.306314798473851}, yticklabels={{$0.0$,$2.5$,$5.0$,$7.5$,$10.0$}}, ytick={{0.0,2.5,5.0,7.5,10.0}}, ytick align={inside}, yticklabel style={font={{\fontsize{24 pt}{31.200000000000003 pt}\selectfont}}, color={rgb,1:red,0.0;green,0.0;blue,0.0}, draw opacity={1.0}, rotate={0.0}}, y grid style={color={rgb,1:red,0.0;green,0.0;blue,0.0}, draw opacity={0.1}, line width={0.5}, solid}, axis y line*={left}, y axis line style={color={rgb,1:red,0.0;green,0.0;blue,0.0}, draw opacity={1.0}, line width={1}, solid}, colorbar={false}]
    \addplot+[line width={0}, draw opacity={0}, fill={rgb,1:red,0.502;green,0.0;blue,0.502}, fill opacity={0.15}, mark={none}, forget plot]
        coordinates {
            (2.0,0.09575604648383114)
            (4.0,0.5269876074404318)
            (8.0,1.5451885745758214)
            (16.0,3.216155949025804)
            (32.0,5.63199669238257)
            (32.0,5.625025413698436)
            (16.0,3.209377310542418)
            (8.0,1.538527914745694)
            (4.0,0.5232429762011946)
            (2.0,0.0938938092700498)
            (2.0,0.09575604648383114)
        }
        ;
    \addplot+[line width={0}, draw opacity={0}, fill={rgb,1:red,0.502;green,0.0;blue,0.502}, fill opacity={0.15}, mark={none}, forget plot]
        coordinates {
            (2.0,0.09575604648383114)
            (4.0,0.5269876074404318)
            (8.0,1.5451885745758214)
            (16.0,3.216155949025804)
            (32.0,5.63199669238257)
            (32.0,5.6389679710667036)
            (16.0,3.2229345875091897)
            (8.0,1.5518492344059487)
            (4.0,0.530732238679669)
            (2.0,0.09761828369761248)
            (2.0,0.09575604648383114)
        }
        ;
    \addplot[color={rgb,1:red,0.502;green,0.0;blue,0.502}, name path={17}, legend image code/.code={{
    \draw[fill={rgb,1:red,0.502;green,0.0;blue,0.502}, fill opacity={0.15}] (0cm,-0.1cm) rectangle (0.6cm,0.1cm);
    }}, draw opacity={1.0}, line width={2}, solid, mark={*}, mark size={3.0 pt}, mark repeat={1}, mark options={color={rgb,1:red,0.0;green,0.0;blue,0.0}, draw opacity={1.0}, fill={rgb,1:red,0.502;green,0.0;blue,0.502}, fill opacity={1.0}, line width={0.75}, rotate={0}, solid}]
        table[row sep={\\}]
        {
            \\
            2.0  0.09575604648383114  \\
            4.0  0.5269876074404318  \\
            8.0  1.5451885745758214  \\
            16.0  3.216155949025804  \\
            32.0  5.63199669238257  \\
        }
        ;
    \addlegendentry {mean $\Delta$}
    \addplot+[line width={0}, draw opacity={0}, fill={rgb,1:red,1.0;green,0.6471;blue,0.0}, fill opacity={0.15}, mark={none}, forget plot]
        coordinates {
            (2.0,4.481140194904158)
            (4.0,5.421210124096172)
            (8.0,6.826017176765849)
            (16.0,8.53918893179165)
            (32.0,11.21320469674185)
            (32.0,10.475796177219667)
            (16.0,7.816189968721643)
            (8.0,5.891628590678088)
            (4.0,4.3751370143332)
            (2.0,3.1108201597496574)
            (2.0,4.481140194904158)
        }
        ;
    \addplot+[line width={0}, draw opacity={0}, fill={rgb,1:red,1.0;green,0.6471;blue,0.0}, fill opacity={0.15}, mark={none}, forget plot]
        coordinates {
            (2.0,4.481140194904158)
            (4.0,5.421210124096172)
            (8.0,6.826017176765849)
            (16.0,8.53918893179165)
            (32.0,11.21320469674185)
            (32.0,11.950613216264031)
            (16.0,9.262187894861656)
            (8.0,7.760405762853611)
            (4.0,6.467283233859144)
            (2.0,5.851460230058658)
            (2.0,4.481140194904158)
        }
        ;
    \addplot[color={rgb,1:red,1.0;green,0.6471;blue,0.0}, name path={18}, legend image code/.code={{
    \draw[fill={rgb,1:red,1.0;green,0.6471;blue,0.0}, fill opacity={0.15}] (0cm,-0.1cm) rectangle (0.6cm,0.1cm);
    }}, draw opacity={1.0}, line width={2}, dashed, mark={diamond*}, mark size={3.0 pt}, mark repeat={1}, mark options={color={rgb,1:red,0.0;green,0.0;blue,0.0}, draw opacity={1.0}, fill={rgb,1:red,1.0;green,0.6471;blue,0.0}, fill opacity={1.0}, line width={0.75}, rotate={0}, solid}]
        table[row sep={\\}]
        {
            \\
            2.0  4.481140194904158  \\
            4.0  5.421210124096172  \\
            8.0  6.826017176765849  \\
            16.0  8.53918893179165  \\
            32.0  11.21320469674185  \\
        }
        ;
    \addlegendentry {max $\Delta$}
\end{axis}
\end{tikzpicture}}
  \resizebox{.32\textwidth}{!}{% Recommended preamble:
% \usetikzlibrary{arrows.meta}
% \usetikzlibrary{backgrounds}
% \usepgfplotslibrary{patchplots}
% \usepgfplotslibrary{fillbetween}
% \pgfplotsset{%
%     layers/standard/.define layer set={%
%         background,axis background,axis grid,axis ticks,axis lines,axis tick labels,pre main,main,axis descriptions,axis foreground%
%     }{
%         grid style={/pgfplots/on layer=axis grid},%
%         tick style={/pgfplots/on layer=axis ticks},%
%         axis line style={/pgfplots/on layer=axis lines},%
%         label style={/pgfplots/on layer=axis descriptions},%
%         legend style={/pgfplots/on layer=axis descriptions},%
%         title style={/pgfplots/on layer=axis descriptions},%
%         colorbar style={/pgfplots/on layer=axis descriptions},%
%         ticklabel style={/pgfplots/on layer=axis tick labels},%
%         axis background@ style={/pgfplots/on layer=axis background},%
%         3d box foreground style={/pgfplots/on layer=axis foreground},%
%     },
% }

\begin{tikzpicture}[/tikz/background rectangle/.style={fill={rgb,1:red,1.0;green,1.0;blue,1.0}, fill opacity={1.0}, draw opacity={1.0}}, show background rectangle]
\begin{axis}[point meta max={nan}, point meta min={nan}, legend cell align={left}, legend columns={1}, title={}, title style={at={{(0.5,1)}}, anchor={south}, font={{\fontsize{26 pt}{33.800000000000004 pt}\selectfont}}, color={rgb,1:red,0.0;green,0.0;blue,0.0}, draw opacity={1.0}, rotate={0.0}, align={center}}, legend style={color={rgb,1:red,0.0;green,0.0;blue,0.0}, draw opacity={1.0}, line width={1}, solid, fill={rgb,1:red,1.0;green,1.0;blue,1.0}, fill opacity={1.0}, text opacity={1.0}, font={{\fontsize{20 pt}{26.0 pt}\selectfont}}, text={rgb,1:red,0.0;green,0.0;blue,0.0}, cells={anchor={center}}, at={(1.02, 1)}, anchor={north west}}, axis background/.style={fill={rgb,1:red,1.0;green,1.0;blue,1.0}, opacity={1.0}}, anchor={north west}, xshift={8.0mm}, yshift={-1.0mm}, width={133.24mm}, height={89.52mm}, scaled x ticks={false}, xlabel={Input dimension M}, x tick style={color={rgb,1:red,0.0;green,0.0;blue,0.0}, opacity={1.0}}, x tick label style={color={rgb,1:red,0.0;green,0.0;blue,0.0}, opacity={1.0}, rotate={0}}, xlabel style={at={(ticklabel cs:0.5)}, anchor=near ticklabel, at={{(ticklabel cs:0.5)}}, anchor={near ticklabel}, font={{\fontsize{28 pt}{36.4 pt}\selectfont}}, color={rgb,1:red,0.0;green,0.0;blue,0.0}, draw opacity={1.0}, rotate={0.0}}, xmode={log}, log basis x={2}, xmajorgrids={true}, xmin={0.90125046261083}, xmax={35.50622310617105}, xticklabels={{$2^{0}$,$2^{2}$,$2^{4}$}}, xtick={{1.0,4.0,16.0}}, xtick align={inside}, xticklabel style={font={{\fontsize{24 pt}{31.200000000000003 pt}\selectfont}}, color={rgb,1:red,0.0;green,0.0;blue,0.0}, draw opacity={1.0}, rotate={0.0}}, x grid style={color={rgb,1:red,0.0;green,0.0;blue,0.0}, draw opacity={0.1}, line width={0.5}, solid}, axis x line*={left}, x axis line style={color={rgb,1:red,0.0;green,0.0;blue,0.0}, draw opacity={1.0}, line width={1}, solid}, scaled y ticks={false}, ylabel={RMSE gap}, y tick style={color={rgb,1:red,0.0;green,0.0;blue,0.0}, opacity={1.0}}, y tick label style={color={rgb,1:red,0.0;green,0.0;blue,0.0}, opacity={1.0}, rotate={0}}, ylabel style={at={(ticklabel cs:0.5)}, anchor=near ticklabel, at={{(ticklabel cs:0.5)}}, anchor={near ticklabel}, font={{\fontsize{28 pt}{36.4 pt}\selectfont}}, color={rgb,1:red,0.0;green,0.0;blue,0.0}, draw opacity={1.0}, rotate={0.0}}, ymajorgrids={true}, ymin={-0.004474252190357464}, ymax={0.15361599186893954}, yticklabels={{$0.00$,$0.05$,$0.10$,$0.15$}}, ytick={{0.0,0.05000000000000001,0.10000000000000002,0.15000000000000002}}, ytick align={inside}, yticklabel style={font={{\fontsize{24 pt}{31.200000000000003 pt}\selectfont}}, color={rgb,1:red,0.0;green,0.0;blue,0.0}, draw opacity={1.0}, rotate={0.0}}, y grid style={color={rgb,1:red,0.0;green,0.0;blue,0.0}, draw opacity={0.1}, line width={0.5}, solid}, axis y line*={left}, y axis line style={color={rgb,1:red,0.0;green,0.0;blue,0.0}, draw opacity={1.0}, line width={1}, solid}, colorbar={false}]
    \addplot[color={rgb,1:red,0.502;green,0.502;blue,0.502}, name path={13}, draw opacity={1.0}, line width={1}, dashed]
        table[row sep={\\}]
        {
            \\
            0.02287633899915038  0.0  \\
            1398.825222916502  0.0  \\
        }
        ;
    \addplot+[line width={0}, draw opacity={0}, fill={rgb,1:red,0.0;green,0.0;blue,1.0}, fill opacity={0.15}, mark={none}, forget plot]
        coordinates {
            (2.0,0.003025485447655671)
            (4.0,0.008898244848943946)
            (8.0,0.06339105130205483)
            (16.0,0.10283648232130839)
            (32.0,0.1137080640600252)
            (32.0,0.07827438844146832)
            (16.0,0.0685086117518614)
            (8.0,0.03847484519293731)
            (4.0,0.004782784080888547)
            (2.0,0.0010083756574806443)
            (2.0,0.003025485447655671)
        }
        ;
    \addplot+[line width={0}, draw opacity={0}, fill={rgb,1:red,0.0;green,0.0;blue,1.0}, fill opacity={0.15}, mark={none}, forget plot]
        coordinates {
            (2.0,0.003025485447655671)
            (4.0,0.008898244848943946)
            (8.0,0.06339105130205483)
            (16.0,0.10283648232130839)
            (32.0,0.1137080640600252)
            (32.0,0.14914173967858207)
            (16.0,0.13716435289075538)
            (8.0,0.08830725741117235)
            (4.0,0.013013705616999344)
            (2.0,0.005042595237830697)
            (2.0,0.003025485447655671)
        }
        ;
    \addplot[color={rgb,1:red,0.0;green,0.0;blue,1.0}, name path={14}, draw opacity={1.0}, line width={2}, solid, mark={*}, mark size={3.0 pt}, mark repeat={1}, mark options={color={rgb,1:red,0.0;green,0.0;blue,0.0}, draw opacity={1.0}, fill={rgb,1:red,0.0;green,0.0;blue,1.0}, fill opacity={1.0}, line width={0.75}, rotate={0}, solid}]
        table[row sep={\\}]
        {
            \\
            2.0  0.003025485447655671  \\
            4.0  0.008898244848943946  \\
            8.0  0.06339105130205483  \\
            16.0  0.10283648232130839  \\
            32.0  0.1137080640600252  \\
        }
        ;
\end{axis}
\end{tikzpicture}}
  \resizebox{.32\textwidth}{!}{% Recommended preamble:
% \usetikzlibrary{arrows.meta}
% \usetikzlibrary{backgrounds}
% \usepgfplotslibrary{patchplots}
% \usepgfplotslibrary{fillbetween}
% \pgfplotsset{%
%     layers/standard/.define layer set={%
%         background,axis background,axis grid,axis ticks,axis lines,axis tick labels,pre main,main,axis descriptions,axis foreground%
%     }{
%         grid style={/pgfplots/on layer=axis grid},%
%         tick style={/pgfplots/on layer=axis ticks},%
%         axis line style={/pgfplots/on layer=axis lines},%
%         label style={/pgfplots/on layer=axis descriptions},%
%         legend style={/pgfplots/on layer=axis descriptions},%
%         title style={/pgfplots/on layer=axis descriptions},%
%         colorbar style={/pgfplots/on layer=axis descriptions},%
%         ticklabel style={/pgfplots/on layer=axis tick labels},%
%         axis background@ style={/pgfplots/on layer=axis background},%
%         3d box foreground style={/pgfplots/on layer=axis foreground},%
%     },
% }

\begin{tikzpicture}[/tikz/background rectangle/.style={fill={rgb,1:red,1.0;green,1.0;blue,1.0}, fill opacity={1.0}, draw opacity={1.0}}, show background rectangle]
\begin{axis}[point meta max={nan}, point meta min={nan}, legend cell align={left}, legend columns={1}, title={}, title style={at={{(0.5,1)}}, anchor={south}, font={{\fontsize{26 pt}{33.800000000000004 pt}\selectfont}}, color={rgb,1:red,0.0;green,0.0;blue,0.0}, draw opacity={1.0}, rotate={0.0}, align={center}}, legend style={color={rgb,1:red,0.0;green,0.0;blue,0.0}, draw opacity={1.0}, line width={1}, solid, fill={rgb,1:red,1.0;green,1.0;blue,1.0}, fill opacity={1.0}, text opacity={1.0}, font={{\fontsize{20 pt}{26.0 pt}\selectfont}}, text={rgb,1:red,0.0;green,0.0;blue,0.0}, cells={anchor={center}}, at={(1.02, 1)}, anchor={north west}}, axis background/.style={fill={rgb,1:red,1.0;green,1.0;blue,1.0}, opacity={1.0}}, anchor={north west}, xshift={8.0mm}, yshift={-1.0mm}, width={133.24mm}, height={89.52mm}, scaled x ticks={false}, xlabel={Input dimension M}, x tick style={color={rgb,1:red,0.0;green,0.0;blue,0.0}, opacity={1.0}}, x tick label style={color={rgb,1:red,0.0;green,0.0;blue,0.0}, opacity={1.0}, rotate={0}}, xlabel style={at={(ticklabel cs:0.5)}, anchor=near ticklabel, at={{(ticklabel cs:0.5)}}, anchor={near ticklabel}, font={{\fontsize{28 pt}{36.4 pt}\selectfont}}, color={rgb,1:red,0.0;green,0.0;blue,0.0}, draw opacity={1.0}, rotate={0.0}}, xmode={log}, log basis x={2}, xmajorgrids={true}, xmin={0.90125046261083}, xmax={35.50622310617105}, xticklabels={{$2^{0}$,$2^{2}$,$2^{4}$}}, xtick={{1.0,4.0,16.0}}, xtick align={inside}, xticklabel style={font={{\fontsize{24 pt}{31.200000000000003 pt}\selectfont}}, color={rgb,1:red,0.0;green,0.0;blue,0.0}, draw opacity={1.0}, rotate={0.0}}, x grid style={color={rgb,1:red,0.0;green,0.0;blue,0.0}, draw opacity={0.1}, line width={0.5}, solid}, axis x line*={left}, x axis line style={color={rgb,1:red,0.0;green,0.0;blue,0.0}, draw opacity={1.0}, line width={1}, solid}, scaled y ticks={false}, ylabel={MNLL gap}, y tick style={color={rgb,1:red,0.0;green,0.0;blue,0.0}, opacity={1.0}}, y tick label style={color={rgb,1:red,0.0;green,0.0;blue,0.0}, opacity={1.0}, rotate={0}}, ylabel style={at={(ticklabel cs:0.5)}, anchor=near ticklabel, at={{(ticklabel cs:0.5)}}, anchor={near ticklabel}, font={{\fontsize{28 pt}{36.4 pt}\selectfont}}, color={rgb,1:red,0.0;green,0.0;blue,0.0}, draw opacity={1.0}, rotate={0.0}}, ymajorgrids={true}, ymin={-0.04622885063855359}, ymax={0.6941106963183583}, yticklabels={{$0.0$,$0.1$,$0.2$,$0.3$,$0.4$,$0.5$,$0.6$}}, ytick={{0.0,0.1,0.2,0.30000000000000004,0.4,0.5,0.6000000000000001}}, ytick align={inside}, yticklabel style={font={{\fontsize{24 pt}{31.200000000000003 pt}\selectfont}}, color={rgb,1:red,0.0;green,0.0;blue,0.0}, draw opacity={1.0}, rotate={0.0}}, y grid style={color={rgb,1:red,0.0;green,0.0;blue,0.0}, draw opacity={0.1}, line width={0.5}, solid}, axis y line*={left}, y axis line style={color={rgb,1:red,0.0;green,0.0;blue,0.0}, draw opacity={1.0}, line width={1}, solid}, colorbar={false}]
    \addplot[color={rgb,1:red,0.502;green,0.502;blue,0.502}, name path={15}, draw opacity={1.0}, line width={1}, dashed]
        table[row sep={\\}]
        {
            \\
            0.02287633899915038  0.0  \\
            1398.825222916502  0.0  \\
        }
        ;
    \addplot+[line width={0}, draw opacity={0}, fill={rgb,1:red,0.0;green,0.0;blue,1.0}, fill opacity={0.15}, mark={none}, forget plot]
        coordinates {
            (2.0,0.07825979277818418)
            (4.0,0.2411908446985143)
            (8.0,0.3602877697570974)
            (16.0,0.37404746849803333)
            (32.0,0.20103196457192066)
            (32.0,-0.025275844592603208)
            (16.0,0.08600034901273051)
            (8.0,0.047417849241786936)
            (4.0,0.08757697703480091)
            (2.0,-0.005146261807565541)
            (2.0,0.07825979277818418)
        }
        ;
    \addplot+[line width={0}, draw opacity={0}, fill={rgb,1:red,0.0;green,0.0;blue,1.0}, fill opacity={0.15}, mark={none}, forget plot]
        coordinates {
            (2.0,0.07825979277818418)
            (4.0,0.2411908446985143)
            (8.0,0.3602877697570974)
            (16.0,0.37404746849803333)
            (32.0,0.20103196457192066)
            (32.0,0.4273397737364445)
            (16.0,0.6620945879833362)
            (8.0,0.6731576902724079)
            (4.0,0.3948047123622277)
            (2.0,0.1616658473639339)
            (2.0,0.07825979277818418)
        }
        ;
    \addplot[color={rgb,1:red,0.0;green,0.0;blue,1.0}, name path={16}, draw opacity={1.0}, line width={2}, solid, mark={*}, mark size={3.0 pt}, mark repeat={1}, mark options={color={rgb,1:red,0.0;green,0.0;blue,0.0}, draw opacity={1.0}, fill={rgb,1:red,0.0;green,0.0;blue,1.0}, fill opacity={1.0}, line width={0.75}, rotate={0}, solid}]
        table[row sep={\\}]
        {
            \\
            2.0  0.07825979277818418  \\
            4.0  0.2411908446985143  \\
            8.0  0.3602877697570974  \\
            16.0  0.37404746849803333  \\
            32.0  0.20103196457192066  \\
        }
        ;
\end{axis}
\end{tikzpicture}}
  \caption{Chain stretch and approximation gap on the sensor network sweep ($C{=}2000$, $D{=}3$, $L{=}2$, $100$ seeds). Left: nearest-neighbor chain stretch $\Delta$ (mean, max) vs.\ input dim $M$. Middle and right: held-out RMSE and MNLL gap vs.\ input dim $M$, as SS-LMC minus exact KM-LMC. Difference is taken within each seed and then averaged (mean $\pm$ std).}
  \label{fig:chain_gap}
\end{figure*}
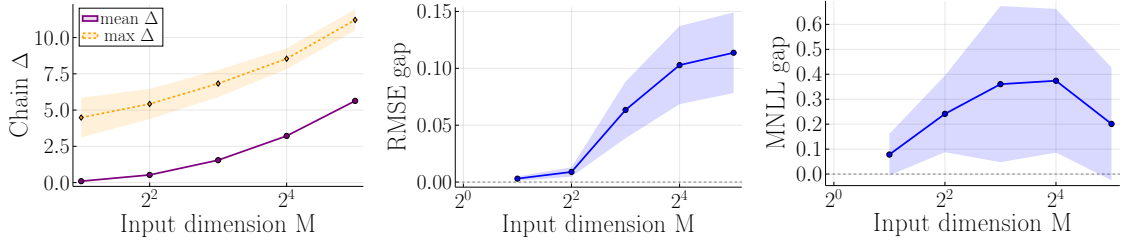

\section{ETT Temporal (\texorpdfstring{$M{=}1$}{M=1}) Case} \label{app:ett_m1}

As a complement to the multidimensional-input study of \cref{sec:ett}, we also forecast ETTh1 in the purely temporal setting: all seven channels are modeled jointly as $D{=}7$ correlated outputs over normalized time, so $M{=}1$ and the candidate chain reduces to the natural temporal ordering. Here SS-LMC and KM-LMC realize the same model, so their posteriors coincide to numerical precision at every window length and the comparison is purely computational. \Cref{fig:ETT_m1} shows that SS-LMC's cost grows linearly in the window length while KM-LMC's covariance restructuring grows cubically, with SS-LMC reaching a $4.4\times$ speed-up at $C{=}4000$. SS-LMC's cost is also invariant to the dropout rate, whereas KM-LMC's falls with sparser data but stays tied to dense factorization. These timings come from a separate run and are comparable within the figure, not against \cref{fig:ETT}.

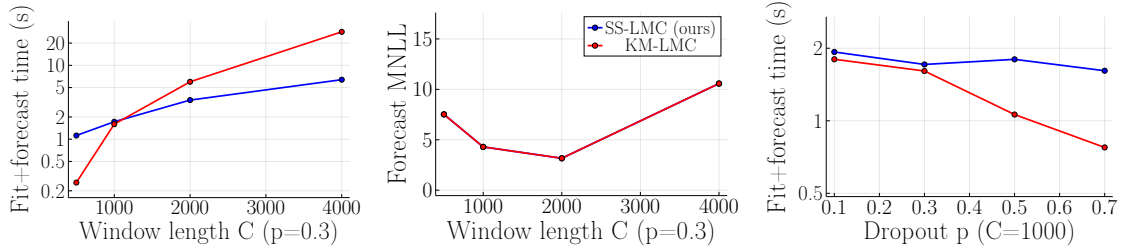
\begin{figure*}[htb]
  \centering
  \resizebox{.32\textwidth}{!}{% Recommended preamble:
% \usetikzlibrary{arrows.meta}
% \usetikzlibrary{backgrounds}
% \usepgfplotslibrary{patchplots}
% \usepgfplotslibrary{fillbetween}
% \pgfplotsset{%
%     layers/standard/.define layer set={%
%         background,axis background,axis grid,axis ticks,axis lines,axis tick labels,pre main,main,axis descriptions,axis foreground%
%     }{
%         grid style={/pgfplots/on layer=axis grid},%
%         tick style={/pgfplots/on layer=axis ticks},%
%         axis line style={/pgfplots/on layer=axis lines},%
%         label style={/pgfplots/on layer=axis descriptions},%
%         legend style={/pgfplots/on layer=axis descriptions},%
%         title style={/pgfplots/on layer=axis descriptions},%
%         colorbar style={/pgfplots/on layer=axis descriptions},%
%         ticklabel style={/pgfplots/on layer=axis tick labels},%
%         axis background@ style={/pgfplots/on layer=axis background},%
%         3d box foreground style={/pgfplots/on layer=axis foreground},%
%     },
% }

\begin{tikzpicture}[/tikz/background rectangle/.style={fill={rgb,1:red,1.0;green,1.0;blue,1.0}, fill opacity={1.0}, draw opacity={1.0}}, show background rectangle]
\begin{axis}[point meta max={nan}, point meta min={nan}, legend cell align={left}, legend columns={1}, title={}, title style={at={{(0.5,1)}}, anchor={south}, font={{\fontsize{26 pt}{33.800000000000004 pt}\selectfont}}, color={rgb,1:red,0.0;green,0.0;blue,0.0}, draw opacity={1.0}, rotate={0.0}, align={center}}, legend style={color={rgb,1:red,0.0;green,0.0;blue,0.0}, draw opacity={1.0}, line width={1}, solid, fill={rgb,1:red,1.0;green,1.0;blue,1.0}, fill opacity={1.0}, text opacity={1.0}, font={{\fontsize{20 pt}{26.0 pt}\selectfont}}, text={rgb,1:red,0.0;green,0.0;blue,0.0}, cells={anchor={center}}, at={(1.02, 1)}, anchor={north west}}, axis background/.style={fill={rgb,1:red,1.0;green,1.0;blue,1.0}, opacity={1.0}}, anchor={north west}, xshift={8.0mm}, yshift={-1.0mm}, width={133.24mm}, height={89.52mm}, scaled x ticks={false}, xlabel={Window length C (p=0.3)}, x tick style={color={rgb,1:red,0.0;green,0.0;blue,0.0}, opacity={1.0}}, x tick label style={color={rgb,1:red,0.0;green,0.0;blue,0.0}, opacity={1.0}, rotate={0}}, xlabel style={at={(ticklabel cs:0.5)}, anchor=near ticklabel, at={{(ticklabel cs:0.5)}}, anchor={near ticklabel}, font={{\fontsize{28 pt}{36.4 pt}\selectfont}}, color={rgb,1:red,0.0;green,0.0;blue,0.0}, draw opacity={1.0}, rotate={0.0}}, xmajorgrids={true}, xmin={395.0}, xmax={4105.0}, xticklabels={{$1000$,$2000$,$3000$,$4000$}}, xtick={{1000.0,2000.0,3000.0,4000.0}}, xtick align={inside}, xticklabel style={font={{\fontsize{24 pt}{31.200000000000003 pt}\selectfont}}, color={rgb,1:red,0.0;green,0.0;blue,0.0}, draw opacity={1.0}, rotate={0.0}}, x grid style={color={rgb,1:red,0.0;green,0.0;blue,0.0}, draw opacity={0.1}, line width={0.5}, solid}, axis x line*={left}, x axis line style={color={rgb,1:red,0.0;green,0.0;blue,0.0}, draw opacity={1.0}, line width={1}, solid}, scaled y ticks={false}, ylabel={Fit+forecast time (s)}, y tick style={color={rgb,1:red,0.0;green,0.0;blue,0.0}, opacity={1.0}}, y tick label style={color={rgb,1:red,0.0;green,0.0;blue,0.0}, opacity={1.0}, rotate={0}}, ylabel style={at={(ticklabel cs:0.5)}, anchor=near ticklabel, at={{(ticklabel cs:0.5)}}, anchor={near ticklabel}, font={{\fontsize{28 pt}{36.4 pt}\selectfont}}, color={rgb,1:red,0.0;green,0.0;blue,0.0}, draw opacity={1.0}, rotate={0.0}}, ymode={log}, log basis y={10}, ymajorgrids={true}, ymin={0.161723746875}, ymax={39.76560559346666}, yticklabels={{0.2,0.5,1,2,5,10,20}}, ytick={{0.2,0.5,1.0,2.0,5.0,10.0,20.0}}, ytick align={inside}, yticklabel style={font={{\fontsize{24 pt}{31.200000000000003 pt}\selectfont}}, color={rgb,1:red,0.0;green,0.0;blue,0.0}, draw opacity={1.0}, rotate={0.0}}, y grid style={color={rgb,1:red,0.0;green,0.0;blue,0.0}, draw opacity={0.1}, line width={0.5}, solid}, axis y line*={left}, y axis line style={color={rgb,1:red,0.0;green,0.0;blue,0.0}, draw opacity={1.0}, line width={1}, solid}, colorbar={false}]
    \addplot[color={rgb,1:red,0.0;green,0.0;blue,1.0}, name path={77}, draw opacity={1.0}, line width={2}, solid, mark={*}, mark size={3.0 pt}, mark repeat={1}, mark options={color={rgb,1:red,0.0;green,0.0;blue,0.0}, draw opacity={1.0}, fill={rgb,1:red,0.0;green,0.0;blue,1.0}, fill opacity={1.0}, line width={0.75}, rotate={0}, solid}]
        table[row sep={\\}]
        {
            \\
            500.0  1.1213144906666666  \\
            1000.0  1.7232759576666667  \\
            2000.0  3.383519186333334  \\
            4000.0  6.3929680190000004  \\
        }
        ;
    \addplot[color={rgb,1:red,1.0;green,0.0;blue,0.0}, name path={78}, draw opacity={1.0}, line width={2}, solid, mark={*}, mark size={3.0 pt}, mark repeat={1}, mark options={color={rgb,1:red,0.0;green,0.0;blue,0.0}, draw opacity={1.0}, fill={rgb,1:red,1.0;green,0.0;blue,0.0}, fill opacity={1.0}, line width={0.75}, rotate={0}, solid}]
        table[row sep={\\}]
        {
            \\
            500.0  0.258757995  \\
            1000.0  1.5987543039999998  \\
            2000.0  5.9865310276666674  \\
            4000.0  28.40400399533333  \\
        }
        ;
\end{axis}
\end{tikzpicture}}
  \resizebox{.32\textwidth}{!}{% Recommended preamble:
% \usetikzlibrary{arrows.meta}
% \usetikzlibrary{backgrounds}
% \usepgfplotslibrary{patchplots}
% \usepgfplotslibrary{fillbetween}
% \pgfplotsset{%
%     layers/standard/.define layer set={%
%         background,axis background,axis grid,axis ticks,axis lines,axis tick labels,pre main,main,axis descriptions,axis foreground%
%     }{
%         grid style={/pgfplots/on layer=axis grid},%
%         tick style={/pgfplots/on layer=axis ticks},%
%         axis line style={/pgfplots/on layer=axis lines},%
%         label style={/pgfplots/on layer=axis descriptions},%
%         legend style={/pgfplots/on layer=axis descriptions},%
%         title style={/pgfplots/on layer=axis descriptions},%
%         colorbar style={/pgfplots/on layer=axis descriptions},%
%         ticklabel style={/pgfplots/on layer=axis tick labels},%
%         axis background@ style={/pgfplots/on layer=axis background},%
%         3d box foreground style={/pgfplots/on layer=axis foreground},%
%     },
% }

\begin{tikzpicture}[/tikz/background rectangle/.style={fill={rgb,1:red,1.0;green,1.0;blue,1.0}, fill opacity={1.0}, draw opacity={1.0}}, show background rectangle]
\begin{axis}[point meta max={nan}, point meta min={nan}, legend cell align={left}, legend columns={1}, title={}, title style={at={{(0.5,1)}}, anchor={south}, font={{\fontsize{26 pt}{33.800000000000004 pt}\selectfont}}, color={rgb,1:red,0.0;green,0.0;blue,0.0}, draw opacity={1.0}, rotate={0.0}, align={center}}, legend style={color={rgb,1:red,0.0;green,0.0;blue,0.0}, draw opacity={1.0}, line width={1}, solid, fill={rgb,1:red,1.0;green,1.0;blue,1.0}, fill opacity={1.0}, text opacity={1.0}, font={{\fontsize{20 pt}{26.0 pt}\selectfont}}, text={rgb,1:red,0.0;green,0.0;blue,0.0}, cells={anchor={center}}, at={(0.98, 0.98)}, anchor={north east}}, axis background/.style={fill={rgb,1:red,1.0;green,1.0;blue,1.0}, opacity={1.0}}, anchor={north west}, xshift={8.0mm}, yshift={-1.0mm}, width={133.24mm}, height={89.52mm}, scaled x ticks={false}, xlabel={Window length C (p=0.3)}, x tick style={color={rgb,1:red,0.0;green,0.0;blue,0.0}, opacity={1.0}}, x tick label style={color={rgb,1:red,0.0;green,0.0;blue,0.0}, opacity={1.0}, rotate={0}}, xlabel style={at={(ticklabel cs:0.5)}, anchor=near ticklabel, at={{(ticklabel cs:0.5)}}, anchor={near ticklabel}, font={{\fontsize{28 pt}{36.4 pt}\selectfont}}, color={rgb,1:red,0.0;green,0.0;blue,0.0}, draw opacity={1.0}, rotate={0.0}}, xmajorgrids={true}, xmin={395.0}, xmax={4105.0}, xticklabels={{$1000$,$2000$,$3000$,$4000$}}, xtick={{1000.0,2000.0,3000.0,4000.0}}, xtick align={inside}, xticklabel style={font={{\fontsize{24 pt}{31.200000000000003 pt}\selectfont}}, color={rgb,1:red,0.0;green,0.0;blue,0.0}, draw opacity={1.0}, rotate={0.0}}, x grid style={color={rgb,1:red,0.0;green,0.0;blue,0.0}, draw opacity={0.1}, line width={0.5}, solid}, axis x line*={left}, x axis line style={color={rgb,1:red,0.0;green,0.0;blue,0.0}, draw opacity={1.0}, line width={1}, solid}, scaled y ticks={false}, ylabel={Forecast MNLL}, y tick style={color={rgb,1:red,0.0;green,0.0;blue,0.0}, opacity={1.0}}, y tick label style={color={rgb,1:red,0.0;green,0.0;blue,0.0}, opacity={1.0}, rotate={0}}, ylabel style={at={(ticklabel cs:0.5)}, anchor=near ticklabel, at={{(ticklabel cs:0.5)}}, anchor={near ticklabel}, font={{\fontsize{28 pt}{36.4 pt}\selectfont}}, color={rgb,1:red,0.0;green,0.0;blue,0.0}, draw opacity={1.0}, rotate={0.0}}, ymajorgrids={true}, ymin={-0.5366320395034325}, ymax={17.61048353693081}, yticklabels={{$0$,$5$,$10$,$15$}}, ytick={{0.0,5.0,10.0,15.0}}, ytick align={inside}, yticklabel style={font={{\fontsize{24 pt}{31.200000000000003 pt}\selectfont}}, color={rgb,1:red,0.0;green,0.0;blue,0.0}, draw opacity={1.0}, rotate={0.0}}, y grid style={color={rgb,1:red,0.0;green,0.0;blue,0.0}, draw opacity={0.1}, line width={0.5}, solid}, axis y line*={left}, y axis line style={color={rgb,1:red,0.0;green,0.0;blue,0.0}, draw opacity={1.0}, line width={1}, solid}, colorbar={false}]
    \addplot[color={rgb,1:red,0.0;green,0.0;blue,1.0}, name path={79}, draw opacity={1.0}, line width={2}, solid, mark={*}, mark size={3.0 pt}, mark repeat={1}, mark options={color={rgb,1:red,0.0;green,0.0;blue,0.0}, draw opacity={1.0}, fill={rgb,1:red,0.0;green,0.0;blue,1.0}, fill opacity={1.0}, line width={0.75}, rotate={0}, solid}]
        table[row sep={\\}]
        {
            \\
            500.0  7.520092067491118  \\
            1000.0  4.292673436562858  \\
            2000.0  3.166860935783584  \\
            4000.0  10.573846884844064  \\
        }
        ;
    \addlegendentry {SS-LMC (ours)}
    \addplot[color={rgb,1:red,1.0;green,0.0;blue,0.0}, name path={80}, draw opacity={1.0}, line width={2}, solid, mark={*}, mark size={3.0 pt}, mark repeat={1}, mark options={color={rgb,1:red,0.0;green,0.0;blue,0.0}, draw opacity={1.0}, fill={rgb,1:red,1.0;green,0.0;blue,0.0}, fill opacity={1.0}, line width={0.75}, rotate={0}, solid}]
        table[row sep={\\}]
        {
            \\
            500.0  7.520092064853254  \\
            1000.0  4.29267343356078  \\
            2000.0  3.1668609352790664  \\
            4000.0  10.573846884380956  \\
        }
        ;
    \addlegendentry {KM-LMC}
\end{axis}
\end{tikzpicture}}
  \resizebox{.32\textwidth}{!}{% Recommended preamble:
% \usetikzlibrary{arrows.meta}
% \usetikzlibrary{backgrounds}
% \usepgfplotslibrary{patchplots}
% \usepgfplotslibrary{fillbetween}
% \pgfplotsset{%
%     layers/standard/.define layer set={%
%         background,axis background,axis grid,axis ticks,axis lines,axis tick labels,pre main,main,axis descriptions,axis foreground%
%     }{
%         grid style={/pgfplots/on layer=axis grid},%
%         tick style={/pgfplots/on layer=axis ticks},%
%         axis line style={/pgfplots/on layer=axis lines},%
%         label style={/pgfplots/on layer=axis descriptions},%
%         legend style={/pgfplots/on layer=axis descriptions},%
%         title style={/pgfplots/on layer=axis descriptions},%
%         colorbar style={/pgfplots/on layer=axis descriptions},%
%         ticklabel style={/pgfplots/on layer=axis tick labels},%
%         axis background@ style={/pgfplots/on layer=axis background},%
%         3d box foreground style={/pgfplots/on layer=axis foreground},%
%     },
% }

\begin{tikzpicture}[/tikz/background rectangle/.style={fill={rgb,1:red,1.0;green,1.0;blue,1.0}, fill opacity={1.0}, draw opacity={1.0}}, show background rectangle]
\begin{axis}[point meta max={nan}, point meta min={nan}, legend cell align={left}, legend columns={1}, title={}, title style={at={{(0.5,1)}}, anchor={south}, font={{\fontsize{26 pt}{33.800000000000004 pt}\selectfont}}, color={rgb,1:red,0.0;green,0.0;blue,0.0}, draw opacity={1.0}, rotate={0.0}, align={center}}, legend style={color={rgb,1:red,0.0;green,0.0;blue,0.0}, draw opacity={1.0}, line width={1}, solid, fill={rgb,1:red,1.0;green,1.0;blue,1.0}, fill opacity={1.0}, text opacity={1.0}, font={{\fontsize{20 pt}{26.0 pt}\selectfont}}, text={rgb,1:red,0.0;green,0.0;blue,0.0}, cells={anchor={center}}, at={(1.02, 1)}, anchor={north west}}, axis background/.style={fill={rgb,1:red,1.0;green,1.0;blue,1.0}, opacity={1.0}}, anchor={north west}, xshift={8.0mm}, yshift={-1.0mm}, width={133.24mm}, height={89.52mm}, scaled x ticks={false}, xlabel={Dropout p (C=1000)}, x tick style={color={rgb,1:red,0.0;green,0.0;blue,0.0}, opacity={1.0}}, x tick label style={color={rgb,1:red,0.0;green,0.0;blue,0.0}, opacity={1.0}, rotate={0}}, xlabel style={at={(ticklabel cs:0.5)}, anchor=near ticklabel, at={{(ticklabel cs:0.5)}}, anchor={near ticklabel}, font={{\fontsize{28 pt}{36.4 pt}\selectfont}}, color={rgb,1:red,0.0;green,0.0;blue,0.0}, draw opacity={1.0}, rotate={0.0}}, xmajorgrids={true}, xmin={0.08199999999999996}, xmax={0.718}, xticklabels={{$0.1$,$0.2$,$0.3$,$0.4$,$0.5$,$0.6$,$0.7$}}, xtick={{0.1,0.2,0.30000000000000004,0.4,0.5,0.6000000000000001,0.7000000000000001}}, xtick align={inside}, xticklabel style={font={{\fontsize{24 pt}{31.200000000000003 pt}\selectfont}}, color={rgb,1:red,0.0;green,0.0;blue,0.0}, draw opacity={1.0}, rotate={0.0}}, x grid style={color={rgb,1:red,0.0;green,0.0;blue,0.0}, draw opacity={0.1}, line width={0.5}, solid}, axis x line*={left}, x axis line style={color={rgb,1:red,0.0;green,0.0;blue,0.0}, draw opacity={1.0}, line width={1}, solid}, scaled y ticks={false}, ylabel={Fit+forecast time (s)}, y tick style={color={rgb,1:red,0.0;green,0.0;blue,0.0}, opacity={1.0}}, y tick label style={color={rgb,1:red,0.0;green,0.0;blue,0.0}, opacity={1.0}, rotate={0}}, ylabel style={at={(ticklabel cs:0.5)}, anchor=near ticklabel, at={{(ticklabel cs:0.5)}}, anchor={near ticklabel}, font={{\fontsize{28 pt}{36.4 pt}\selectfont}}, color={rgb,1:red,0.0;green,0.0;blue,0.0}, draw opacity={1.0}, rotate={0.0}}, ymode={log}, log basis y={10}, ymajorgrids={true}, ymin={0.484622055}, ymax={2.702002673533334}, yticklabels={{0.5,1,2}}, ytick={{0.5,1.0,2.0}}, ytick align={inside}, yticklabel style={font={{\fontsize{24 pt}{31.200000000000003 pt}\selectfont}}, color={rgb,1:red,0.0;green,0.0;blue,0.0}, draw opacity={1.0}, rotate={0.0}}, y grid style={color={rgb,1:red,0.0;green,0.0;blue,0.0}, draw opacity={0.1}, line width={0.5}, solid}, axis y line*={left}, y axis line style={color={rgb,1:red,0.0;green,0.0;blue,0.0}, draw opacity={1.0}, line width={1}, solid}, colorbar={false}]
    \addplot[color={rgb,1:red,0.0;green,0.0;blue,1.0}, name path={83}, draw opacity={1.0}, line width={2}, solid, mark={*}, mark size={3.0 pt}, mark repeat={1}, mark options={color={rgb,1:red,0.0;green,0.0;blue,0.0}, draw opacity={1.0}, fill={rgb,1:red,0.0;green,0.0;blue,1.0}, fill opacity={1.0}, line width={0.75}, rotate={0}, solid}]
        table[row sep={\\}]
        {
            \\
            0.1  1.930001909666667  \\
            0.3  1.7130187123333334  \\
            0.5  1.7986713046666665  \\
            0.7  1.6132579053333334  \\
        }
        ;
    \addplot[color={rgb,1:red,1.0;green,0.0;blue,0.0}, name path={84}, draw opacity={1.0}, line width={2}, solid, mark={*}, mark size={3.0 pt}, mark repeat={1}, mark options={color={rgb,1:red,0.0;green,0.0;blue,0.0}, draw opacity={1.0}, fill={rgb,1:red,1.0;green,0.0;blue,0.0}, fill opacity={1.0}, line width={0.75}, rotate={0}, solid}]
        table[row sep={\\}]
        {
            \\
            0.1  1.7998955779999999  \\
            0.3  1.609248048666667  \\
            0.5  1.0617279973333333  \\
            0.7  0.775395288  \\
        }
        ;
\end{axis}
\end{tikzpicture}}
  \caption{ETTh1 temporal forecasting ($M{=}1$ natural ordering, $D{=}7$, $L{=}4$, $3$ seeds): train on the first half of a length-$C$ window, forecast the held-out second half under per-(timestamp, output) dropout. At $M{=}1$ the SS-LMC and KM-LMC posteriors are identical, so the comparison is computational. Left: fit+forecast wall-clock vs.\ $C$ at dropout $p{=}0.3$ (log scale). Middle: held-out MNLL vs.\ $C$. Right: time vs.\ $p$ at $C{=}1000$ (log scale).}
  \label{fig:ETT_m1}
\end{figure*}

\section{Accuracy--Cost Frontier} \label{app:pareto}

Held-out score against fit+forecast wall-clock at fixed window length separates the two accuracy measures of \cref{sec:ett}. Both windows shown are ones where exact KM-LMC still fits in memory, so every panel carries a complete four-method frontier. On RMSE (\cref{fig:pareto}, top row) the mean ordering shifts with the window: at $C{=}1000$ SVGP-LMC is on average both cheaper and more accurate than KM-LMC and NNGP-LMC, whereas at $C{=}2000$ the four means form a monotone frontier. The accuracy differences are small against the seed spread, however: at $C{=}2000$ the entire $0.18$ RMSE span from SS-LMC to KM-LMC is half the $\pm 0.34$ standard deviation across seeds, bought at a $367\times$ cost ratio, and NNGP-LMC reaches KM-LMC's mean RMSE at $15\times$ less time. On MNLL (bottom row) SVGP-LMC and SS-LMC hold the frontier at both windows, with KM-LMC and NNGP-LMC both slower and less well calibrated on average.
%Panels at different $C$ are separate frontiers rather than one curve: the held-out block moves with $C$, so error levels are not comparable across $C$.

\begin{figure*}[htb]
  \centering
  \resizebox{.82\textwidth}{!}{\input{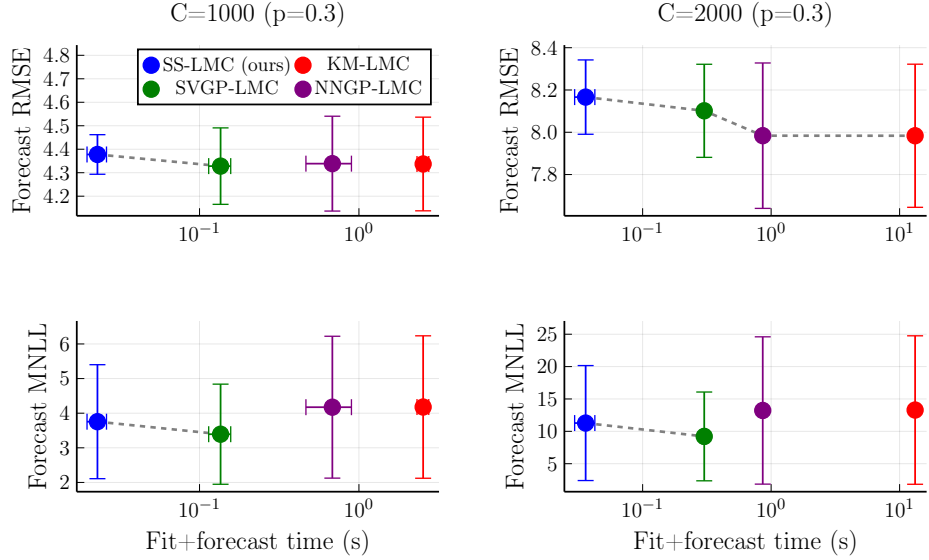}}
  \caption{Held-out score against fit+forecast wall-clock on ETTh1 ($M{=}3$, $D{=}4$, $L{=}3$, dropout $p{=}0.3$, $10$ seeds). Lower-left is better. Top row: RMSE; bottom row: MNLL. Columns are $C{=}1000$ and $C{=}2000$. Dashed line is Pareto front. Bars are $\pm 1$ standard deviation over seeds.}
  \label{fig:pareto}
  \vspace{-10pt}
\end{figure*}

\section{Sensitivity to the Chain Starting Point} \label{app:chain_start}

\cref{prop:chain_start} bounds the worst-case variation across starting points. In this experiment, we measure the variation over $20$ starts. Within a seed, the dropout mask and train/test split are fixed on physical points before ordering. $W$ is fixed, so every start sees the same observations. Per start we take the median over $10$ seeds, then the standard deviation across starts. Since that is itself an estimated dispersion, the error bars are a bootstrap $95\%$ confidence interval for it (relative standard error $\approx\!16\%$ at $20$ starts), so only well-separated points differ.

Sensitivity is governed by input dimension, not candidate-set size. Against $M$ (\cref{fig:chain_start}, left) the spread is flat at $2.5$--$4.0\times10^{-4}$ for $M \le 8$ and steps up roughly threefold at $M \ge 16$ mirroring the $\varepsilon_\pi = \mathcal{O}(C^{1-1/M})$ compression of $M$-dimensional geometry into one chain coordinate. Against $C$ (right) the intervals overlap, so the spread is flat in $C$. The relative effect stays negligible: on ETTh1 the coefficient of variation of held-out RMSE across starts is $\le 0.09\%$ at every $C$. MNLL, not plotted, is somewhat more sensitive, growing four- to sevenfold over the same range of $M$. Three checks support reading this as real but harmless path dependence: the permutation-invariant KM-LMC control varies $10$--$50\times$ less. The even, random and farthest-point start schemes give the same chain diversity and spread, so a bad chain cannot be engineered from a bad start. $\varepsilon_\pi$ does not predict per-start error, its bootstrap correlation intervals remain close to zero. The observed variation therefore sits far inside the envelope of \cref{prop:chain_start}, which holds but is conservative.

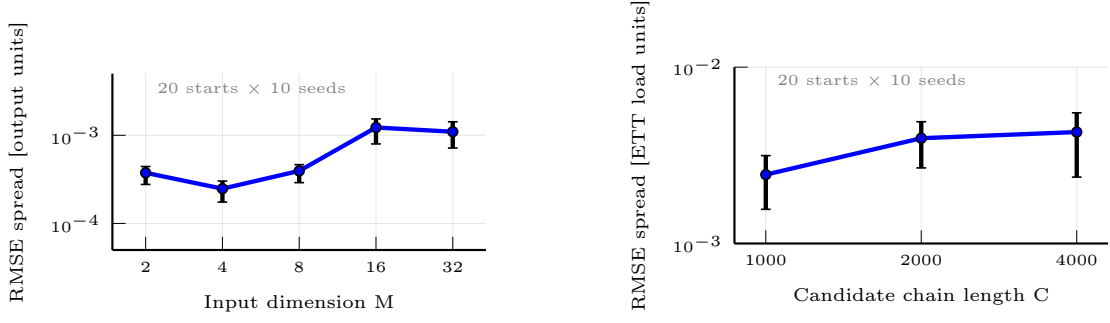
\begin{figure*}[htb]
  \centering
  \resizebox{.46\textwidth}{!}{% Recommended preamble:
% \usetikzlibrary{arrows.meta}
% \usetikzlibrary{backgrounds}
% \usepgfplotslibrary{patchplots}
% \usepgfplotslibrary{fillbetween}
% \pgfplotsset{%
%     layers/standard/.define layer set={%
%         background,axis background,axis grid,axis ticks,axis lines,axis tick labels,pre main,main,axis descriptions,axis foreground%
%     }{
%         grid style={/pgfplots/on layer=axis grid},%
%         tick style={/pgfplots/on layer=axis ticks},%
%         axis line style={/pgfplots/on layer=axis lines},%
%         label style={/pgfplots/on layer=axis descriptions},%
%         legend style={/pgfplots/on layer=axis descriptions},%
%         title style={/pgfplots/on layer=axis descriptions},%
%         colorbar style={/pgfplots/on layer=axis descriptions},%
%         ticklabel style={/pgfplots/on layer=axis tick labels},%
%         axis background@ style={/pgfplots/on layer=axis background},%
%         3d box foreground style={/pgfplots/on layer=axis foreground},%
%     },
% }

\begin{tikzpicture}[/tikz/background rectangle/.style={fill={rgb,1:red,1.0;green,1.0;blue,1.0}, fill opacity={1.0}, draw opacity={1.0}}, show background rectangle]
\begin{axis}[point meta max={nan}, point meta min={nan}, legend cell align={left}, legend columns={1}, title={}, title style={at={{(0.5,1)}}, anchor={south}, font={{\fontsize{6.3 pt}{8.19 pt}\selectfont}}, color={rgb,1:red,0.0;green,0.0;blue,0.0}, draw opacity={1.0}, rotate={0.0}, align={center}}, legend style={color={rgb,1:red,0.0;green,0.0;blue,0.0}, draw opacity={1.0}, line width={0.7}, solid, fill={rgb,1:red,1.0;green,1.0;blue,1.0}, fill opacity={1.0}, text opacity={1.0}, font={{\fontsize{6.3 pt}{8.19 pt}\selectfont}}, text={rgb,1:red,0.0;green,0.0;blue,0.0}, cells={anchor={center}}, at={(1.02, 1)}, anchor={north west}}, axis background/.style={fill={rgb,1:red,1.0;green,1.0;blue,1.0}, opacity={1.0}}, anchor={north west}, xshift={2.0mm}, yshift={-1.0mm}, width={58.18248114mm}, height={35.81285286056942mm}, scaled x ticks={false}, xlabel={Input dimension M}, x tick style={color={rgb,1:red,0.0;green,0.0;blue,0.0}, opacity={1.0}}, x tick label style={color={rgb,1:red,0.0;green,0.0;blue,0.0}, opacity={1.0}, rotate={0}}, xlabel style={at={(ticklabel cs:0.5)}, anchor=near ticklabel, at={{(ticklabel cs:0.5)}}, anchor={near ticklabel}, font={{\fontsize{6.3 pt}{8.19 pt}\selectfont}}, color={rgb,1:red,0.0;green,0.0;blue,0.0}, draw opacity={1.0}, rotate={0.0}}, xmode={log}, log basis x={2}, xmajorgrids={true}, xmin={1.4814814814814814}, xmax={43.2}, xticklabels={{2,4,8,16,32}}, xtick={{2.0,4.0,8.0,16.0,32.0}}, xtick align={inside}, xticklabel style={font={{\fontsize{4.8999999999999995 pt}{6.369999999999999 pt}\selectfont}}, color={rgb,1:red,0.0;green,0.0;blue,0.0}, draw opacity={1.0}, rotate={0.0}}, x grid style={color={rgb,1:red,0.0;green,0.0;blue,0.0}, draw opacity={0.1}, line width={0.35}, solid}, axis x line*={left}, x axis line style={color={rgb,1:red,0.0;green,0.0;blue,0.0}, draw opacity={1.0}, line width={0.7}, solid}, scaled y ticks={false}, ylabel={RMSE spread [output units]}, y tick style={color={rgb,1:red,0.0;green,0.0;blue,0.0}, opacity={1.0}}, y tick label style={color={rgb,1:red,0.0;green,0.0;blue,0.0}, opacity={1.0}, rotate={0}}, ylabel style={at={(ticklabel cs:0.5)}, anchor=near ticklabel, at={{(ticklabel cs:0.5)}}, anchor={near ticklabel}, font={{\fontsize{6.3 pt}{8.19 pt}\selectfont}}, color={rgb,1:red,0.0;green,0.0;blue,0.0}, draw opacity={1.0}, rotate={0.0}}, ymode={log}, log basis y={10}, ymajorgrids={true}, ymin={5.0e-5}, ymax={0.005}, yticklabels={{$10^{-4}$,$10^{-3}$}}, ytick={{0.0001,0.001}}, ytick align={inside}, yticklabel style={font={{\fontsize{4.8999999999999995 pt}{6.369999999999999 pt}\selectfont}}, color={rgb,1:red,0.0;green,0.0;blue,0.0}, draw opacity={1.0}, rotate={0.0}}, y grid style={color={rgb,1:red,0.0;green,0.0;blue,0.0}, draw opacity={0.1}, line width={0.35}, solid}, axis y line*={left}, y axis line style={color={rgb,1:red,0.0;green,0.0;blue,0.0}, draw opacity={1.0}, line width={0.7}, solid}, colorbar={false}]
    \addplot[color={rgb,1:red,0.0;green,0.0;blue,1.0}, name path={2}, draw opacity={1.0}, line width={1.4}, solid, mark={*}, mark size={1.5749999999999997 pt}, mark repeat={1}, mark options={color={rgb,1:red,0.0;green,0.0;blue,0.0}, draw opacity={1.0}, fill={rgb,1:red,0.0;green,0.0;blue,1.0}, fill opacity={1.0}, line width={0.5249999999999999}, rotate={0}, solid}]
        table[row sep={\\}]
        {
            \\
            2.0  0.00037541274181551905  \\
            4.0  0.0002479157441503319  \\
            8.0  0.00039455162362158773  \\
            16.0  0.0012256624861658834  \\
            32.0  0.001096878063435457  \\
        }
        ;
    \addplot[color={rgb,1:red,0.0;green,0.0;blue,0.0}, name path={3}, draw opacity={1.0}, line width={1.4}, solid, mark={-}, mark size={1.5749999999999997 pt}, mark repeat={1}, mark options={color={rgb,1:red,0.0;green,0.0;blue,0.0}, draw opacity={1.0}, fill={rgb,1:red,0.0;green,0.0;blue,0.0}, fill opacity={1.0}, line width={0.5249999999999999}, rotate={0}, solid}]
        table[row sep={\\}]
        {
            \\
            2.0  0.0002770975633668898  \\
            2.0  0.0004423976638862547  \\
        }
        ;
    \addplot[color={rgb,1:red,0.0;green,0.0;blue,0.0}, name path={3}, draw opacity={1.0}, line width={1.4}, solid, mark={-}, mark size={1.5749999999999997 pt}, mark repeat={1}, mark options={color={rgb,1:red,0.0;green,0.0;blue,0.0}, draw opacity={1.0}, fill={rgb,1:red,0.0;green,0.0;blue,0.0}, fill opacity={1.0}, line width={0.5249999999999999}, rotate={0}, solid}]
        table[row sep={\\}]
        {
            \\
            4.0  0.000174747886034784  \\
            4.0  0.00030269636147037114  \\
        }
        ;
    \addplot[color={rgb,1:red,0.0;green,0.0;blue,0.0}, name path={3}, draw opacity={1.0}, line width={1.4}, solid, mark={-}, mark size={1.5749999999999997 pt}, mark repeat={1}, mark options={color={rgb,1:red,0.0;green,0.0;blue,0.0}, draw opacity={1.0}, fill={rgb,1:red,0.0;green,0.0;blue,0.0}, fill opacity={1.0}, line width={0.5249999999999999}, rotate={0}, solid}]
        table[row sep={\\}]
        {
            \\
            8.0  0.00029023793746487666  \\
            8.0  0.00046448292414479455  \\
        }
        ;
    \addplot[color={rgb,1:red,0.0;green,0.0;blue,0.0}, name path={3}, draw opacity={1.0}, line width={1.4}, solid, mark={-}, mark size={1.5749999999999997 pt}, mark repeat={1}, mark options={color={rgb,1:red,0.0;green,0.0;blue,0.0}, draw opacity={1.0}, fill={rgb,1:red,0.0;green,0.0;blue,0.0}, fill opacity={1.0}, line width={0.5249999999999999}, rotate={0}, solid}]
        table[row sep={\\}]
        {
            \\
            16.0  0.0007978163656489447  \\
            16.0  0.0015331483172271055  \\
        }
        ;
    \addplot[color={rgb,1:red,0.0;green,0.0;blue,0.0}, name path={3}, draw opacity={1.0}, line width={1.4}, solid, mark={-}, mark size={1.5749999999999997 pt}, mark repeat={1}, mark options={color={rgb,1:red,0.0;green,0.0;blue,0.0}, draw opacity={1.0}, fill={rgb,1:red,0.0;green,0.0;blue,0.0}, fill opacity={1.0}, line width={0.5249999999999999}, rotate={0}, solid}]
        table[row sep={\\}]
        {
            \\
            32.0  0.0007165816547794108  \\
            32.0  0.0014223104137001132  \\
        }
        ;
    \node[right, , color={rgb,1:red,0.502;green,0.502;blue,0.502}, draw opacity={1.0}, rotate={0.0}, font={{\fontsize{4.8999999999999995 pt}{6.369999999999999 pt}\selectfont}}]  at (axis cs:2.0,0.0033333333333333335) {20 starts $\times$ 10 seeds};
\end{axis}
\end{tikzpicture}}
  \hfill
  \resizebox{.46\textwidth}{!}{% Recommended preamble:
% \usetikzlibrary{arrows.meta}
% \usetikzlibrary{backgrounds}
% \usepgfplotslibrary{patchplots}
% \usepgfplotslibrary{fillbetween}
% \pgfplotsset{%
%     layers/standard/.define layer set={%
%         background,axis background,axis grid,axis ticks,axis lines,axis tick labels,pre main,main,axis descriptions,axis foreground%
%     }{
%         grid style={/pgfplots/on layer=axis grid},%
%         tick style={/pgfplots/on layer=axis ticks},%
%         axis line style={/pgfplots/on layer=axis lines},%
%         label style={/pgfplots/on layer=axis descriptions},%
%         legend style={/pgfplots/on layer=axis descriptions},%
%         title style={/pgfplots/on layer=axis descriptions},%
%         colorbar style={/pgfplots/on layer=axis descriptions},%
%         ticklabel style={/pgfplots/on layer=axis tick labels},%
%         axis background@ style={/pgfplots/on layer=axis background},%
%         3d box foreground style={/pgfplots/on layer=axis foreground},%
%     },
% }

\begin{tikzpicture}[/tikz/background rectangle/.style={fill={rgb,1:red,1.0;green,1.0;blue,1.0}, fill opacity={1.0}, draw opacity={1.0}}, show background rectangle]
\begin{axis}[point meta max={nan}, point meta min={nan}, legend cell align={left}, legend columns={1}, title={}, title style={at={{(0.5,1)}}, anchor={south}, font={{\fontsize{6.3 pt}{8.19 pt}\selectfont}}, color={rgb,1:red,0.0;green,0.0;blue,0.0}, draw opacity={1.0}, rotate={0.0}, align={center}}, legend style={color={rgb,1:red,0.0;green,0.0;blue,0.0}, draw opacity={1.0}, line width={0.7}, solid, fill={rgb,1:red,1.0;green,1.0;blue,1.0}, fill opacity={1.0}, text opacity={1.0}, font={{\fontsize{6.3 pt}{8.19 pt}\selectfont}}, text={rgb,1:red,0.0;green,0.0;blue,0.0}, cells={anchor={center}}, at={(1.02, 1)}, anchor={north west}}, axis background/.style={fill={rgb,1:red,1.0;green,1.0;blue,1.0}, opacity={1.0}}, anchor={north west}, xshift={2.0mm}, yshift={-1.0mm}, width={58.18248114mm}, height={35.81285286056942mm}, scaled x ticks={false}, xlabel={Candidate chain length C}, x tick style={color={rgb,1:red,0.0;green,0.0;blue,0.0}, opacity={1.0}}, x tick label style={color={rgb,1:red,0.0;green,0.0;blue,0.0}, opacity={1.0}, rotate={0}}, xlabel style={at={(ticklabel cs:0.5)}, anchor=near ticklabel, at={{(ticklabel cs:0.5)}}, anchor={near ticklabel}, font={{\fontsize{6.3 pt}{8.19 pt}\selectfont}}, color={rgb,1:red,0.0;green,0.0;blue,0.0}, draw opacity={1.0}, rotate={0.0}}, xmode={log}, log basis x={10}, xmajorgrids={true}, xmin={869.5652173913044}, xmax={4600.0}, xticklabels={{1000,2000,4000}}, xtick={{1000.0,2000.0,4000.0}}, xtick align={inside}, xticklabel style={font={{\fontsize{4.8999999999999995 pt}{6.369999999999999 pt}\selectfont}}, color={rgb,1:red,0.0;green,0.0;blue,0.0}, draw opacity={1.0}, rotate={0.0}}, x grid style={color={rgb,1:red,0.0;green,0.0;blue,0.0}, draw opacity={0.1}, line width={0.35}, solid}, axis x line*={left}, x axis line style={color={rgb,1:red,0.0;green,0.0;blue,0.0}, draw opacity={1.0}, line width={0.7}, solid}, scaled y ticks={false}, ylabel={RMSE spread [ETT load units]}, y tick style={color={rgb,1:red,0.0;green,0.0;blue,0.0}, opacity={1.0}}, y tick label style={color={rgb,1:red,0.0;green,0.0;blue,0.0}, opacity={1.0}, rotate={0}}, ylabel style={at={(ticklabel cs:0.5)}, anchor=near ticklabel, at={{(ticklabel cs:0.5)}}, anchor={near ticklabel}, font={{\fontsize{6.3 pt}{8.19 pt}\selectfont}}, color={rgb,1:red,0.0;green,0.0;blue,0.0}, draw opacity={1.0}, rotate={0.0}}, ymode={log}, log basis y={10}, ymajorgrids={true}, ymin={0.001}, ymax={0.01}, yticklabels={{$10^{-3}$,$10^{-2}$}}, ytick={{0.001,0.01}}, ytick align={inside}, yticklabel style={font={{\fontsize{4.8999999999999995 pt}{6.369999999999999 pt}\selectfont}}, color={rgb,1:red,0.0;green,0.0;blue,0.0}, draw opacity={1.0}, rotate={0.0}}, y grid style={color={rgb,1:red,0.0;green,0.0;blue,0.0}, draw opacity={0.1}, line width={0.35}, solid}, axis y line*={left}, y axis line style={color={rgb,1:red,0.0;green,0.0;blue,0.0}, draw opacity={1.0}, line width={0.7}, solid}, colorbar={false}]
    \addplot[color={rgb,1:red,0.0;green,0.0;blue,1.0}, name path={4}, draw opacity={1.0}, line width={1.4}, solid, mark={*}, mark size={1.5749999999999997 pt}, mark repeat={1}, mark options={color={rgb,1:red,0.0;green,0.0;blue,0.0}, draw opacity={1.0}, fill={rgb,1:red,0.0;green,0.0;blue,1.0}, fill opacity={1.0}, line width={0.5249999999999999}, rotate={0}, solid}]
        table[row sep={\\}]
        {
            \\
            1000.0  0.0024580078471197306  \\
            2000.0  0.0039544447518547994  \\
            4000.0  0.00428457164997435  \\
        }
        ;
    \addplot[color={rgb,1:red,0.0;green,0.0;blue,0.0}, name path={5}, draw opacity={1.0}, line width={1.4}, solid, mark={-}, mark size={1.5749999999999997 pt}, mark repeat={1}, mark options={color={rgb,1:red,0.0;green,0.0;blue,0.0}, draw opacity={1.0}, fill={rgb,1:red,0.0;green,0.0;blue,0.0}, fill opacity={1.0}, line width={0.5249999999999999}, rotate={0}, solid}]
        table[row sep={\\}]
        {
            \\
            1000.0  0.001562729376039964  \\
            1000.0  0.003150485686010693  \\
        }
        ;
    \addplot[color={rgb,1:red,0.0;green,0.0;blue,0.0}, name path={5}, draw opacity={1.0}, line width={1.4}, solid, mark={-}, mark size={1.5749999999999997 pt}, mark repeat={1}, mark options={color={rgb,1:red,0.0;green,0.0;blue,0.0}, draw opacity={1.0}, fill={rgb,1:red,0.0;green,0.0;blue,0.0}, fill opacity={1.0}, line width={0.5249999999999999}, rotate={0}, solid}]
        table[row sep={\\}]
        {
            \\
            2000.0  0.002682676951433166  \\
            2000.0  0.0049023138445093385  \\
        }
        ;
    \addplot[color={rgb,1:red,0.0;green,0.0;blue,0.0}, name path={5}, draw opacity={1.0}, line width={1.4}, solid, mark={-}, mark size={1.5749999999999997 pt}, mark repeat={1}, mark options={color={rgb,1:red,0.0;green,0.0;blue,0.0}, draw opacity={1.0}, fill={rgb,1:red,0.0;green,0.0;blue,0.0}, fill opacity={1.0}, line width={0.5249999999999999}, rotate={0}, solid}]
        table[row sep={\\}]
        {
            \\
            4000.0  0.0023791611214857553  \\
            4000.0  0.005508481466493426  \\
        }
        ;
    \node[right, , color={rgb,1:red,0.502;green,0.502;blue,0.502}, draw opacity={1.0}, rotate={0.0}, font={{\fontsize{4.8999999999999995 pt}{6.369999999999999 pt}\selectfont}}]  at (axis cs:1000.0,0.008333333333333333) {20 starts $\times$ 10 seeds};
\end{axis}
\end{tikzpicture}}
  \caption{Spread of SS-LMC's held-out RMSE across $20$ chain starting points (median over $10$ seeds per start, then standard deviation across starts). Bars are a bootstrap $95\%$ confidence interval for that standard deviation. Left: vs.\ input dimension $M$ on the synthetic sensor network ($C{=}2000$, $D{=}3$, $L{=}2$). Right: vs.\ candidate-set size $C$ on ETTh1 ($M{=}3$, $D{=}4$, $L{=}3$, $p{=}0.3$).}
  \label{fig:chain_start}
  \vspace{-10pt}
\end{figure*}

\section{Proof of Lemma 1} \label{app:lemma1}
\begin{proof}
  Matérn-class kernels are differentiable away from the origin with bounded derivative on any compact range of distances, hence Lipschitz with a finite constant $\alpha_l$ on the range induced by $X$ and the chain $\pi$. For Matérn-3/2 in particular, the maximum of $|\kappa_l'(d)|$ is attained at $d = \ell_l/\sqrt 3$, giving $\alpha_l = \sqrt 3\,\gamma_l^2/(e\,\ell_l)$. Other smoothness orders give analogous closed forms.
  Lipschitz continuity then upper bounds the difference in kernel entries
  \begin{equation}
    \big|\left[K_l - K_{\pi,l}\right]_{ij}\big| = \big|\kappa_l(\|x_i-x_j\|) - \kappa_l(|s_{\pi(i)} - s_{\pi(j)}|)\big| \leq \alpha_l\, \big|\|x_i-x_j\| - |s_{\pi(i)} - s_{\pi(j)}|\big| \ .
  \end{equation}
  This implies $\|K_l - K_{\pi,l}\|_{\max} \leq \alpha_l\, \max_{ij} \big|\|x_i-x_j\| - |s_{\pi(i)} - s_{\pi(j)}|\big| = \alpha_l\, \varepsilon_{\pi}$.
  For any $C\times C$ matrix $A$, $\|A\|_2 \leq \sqrt{\|A\|_1\|A\|_{\infty}} \leq C\|A\|_{\max}$ \citep{horn_matrix_2012}. Hence $\|K_l - K_{\pi,l}\|_2 \leq C\,\alpha_l\,\varepsilon_{\pi}$.
\end{proof}

\section{Lemma 2 and Its Proof} \label{app:lemma2}

\begin{lemma}
  \label{lem:posterior_perturbation}
  Let $K, K_{\pi} \succeq 0$ and $\eta = \|K - K_{\pi}\|_2$. Under scalar observation noise $\sigma_n^2 I$, the GP posterior parameters obtained from $K$ and from $K_{\pi}$ satisfy
  \begin{align}
    \|\mu-\mu_{\pi}\|_2
     & \leq
    \left(
    \frac{\eta}{\sigma_n^2}
    +\frac{(\|K\|_2+\eta)\eta}{\sigma_n^4}
    \right)\|y\|_2 \\
    \|\Sigma-\Sigma_{\pi}\|_2
     & \leq
    \eta
    +\frac{\eta(\|K\|_2+\|K_{\pi}\|_2)}{\sigma_n^2}
    +\frac{\eta\|K\|_2\|K_{\pi}\|_2}{\sigma_n^4}.
  \end{align}
\end{lemma}

\begin{proof}
  Let $B = K + \sigma_n^2 I$, $B_{\pi} = K_{\pi} + \sigma_n^2 I$.
  Since $K,K_{\pi}\succeq0$, we have
  $B,B_{\pi}\succeq\sigma_n^2 I$, so
  $\|B^{-1}\|_2,\|B_{\pi}^{-1}\|_2\leq\sigma_n^{-2}$.
  The resolvent identity
  \begin{equation}
    B^{-1}-B_{\pi}^{-1}
    =B^{-1}(B_{\pi}-B)B_{\pi}^{-1}
    =B^{-1}(K_{\pi}-K)B_{\pi}^{-1}
  \end{equation}
  gives $\|B^{-1}-B_{\pi}^{-1}\|_2\leq\eta\sigma_n^{-4}$.
  Together with submultiplicativity gives
  \begin{equation}
    \|B^{-1}-B_{\pi}^{-1}\|_2
    \leq \|B^{-1}\|_2\,\|K-K_{\pi}\|_2\,\|B_{\pi}^{-1}\|_2
    \leq \eta\,\sigma_n^{-4}.
  \end{equation}
  Adding and subtracting $K_{\pi}B^{-1}y$ inside
  $\mu-\mu_{\pi}=KB^{-1}y-K_{\pi}B_{\pi}^{-1}y$ splits the difference into a
  prior-difference term and an inverse-difference term:
  \begin{equation}
    \mu-\mu_{\pi}
    =(K-K_{\pi})\,B^{-1}\,y
    \;+\;K_{\pi}\,(B^{-1}-B_{\pi}^{-1})\,y.
  \end{equation}
  Taking norms, applying submultiplicativity together with $\|B^{-1}\|_2\leq
    \sigma_n^{-2}$ and the resolvent bound above, and using
  $\|K_{\pi}\|_2\leq\|K\|_2+\eta$,
  \begin{equation}
    \|\mu-\mu_{\pi}\|_2
    \leq
    \frac{\eta}{\sigma_n^2}\,\|y\|_2
    +
    \frac{(\|K\|_2+\eta)\,\eta}{\sigma_n^4}\,\|y\|_2.
  \end{equation}
  Note that
  $\Sigma-\Sigma_{\pi}=(K-K_{\pi})-\bigl(KB^{-1}K-K_{\pi}B_{\pi}^{-1}K_{\pi}\bigr)$.
  If the second piece is expanded by inserting two telescoping terms,
  $\pm K_{\pi}B^{-1}K$ and $\pm K_{\pi}B_{\pi}^{-1}K$, then we have:
  \begin{equation}
    KB^{-1}K-K_{\pi}B_{\pi}^{-1}K_{\pi}
    =(K-K_{\pi})B^{-1}\,K
    + K_{\pi}\,(B^{-1}-B_{\pi}^{-1})\,K
    + K_{\pi}\,B_{\pi}^{-1}\,(K-K_{\pi}).
  \end{equation}
  Taking norms termwise and using the same norm bounds,
  \begin{equation}
    \|\Sigma-\Sigma_{\pi}\|_2
    \leq
    %  \eta
    %  + \frac{\eta\,\|K\|_2}{\sigma_n^2}
    %  + \frac{\|K_{\pi}\|_2\,\eta\,\|K\|_2}{\sigma_n^4}
    %  + \frac{\|K_{\pi}\|_2\,\eta}{\sigma_n^2} 
    %  = 
    \eta
    +\frac{\eta\,(\|K\|_2+\|K_{\pi}\|_2)}{\sigma_n^2}
    + \frac{\eta\,\|K\|_2\,\|K_{\pi}\|_2}{\sigma_n^4}.
  \end{equation}
\end{proof}

\section{Proof of Theorem 1} \label{app:thm1}

\begin{proof}
  Applying \cref{lem:chain_kernel_bound} to each Matérn-class latent kernel gives $\|K_l - K_{\pi,l}\|_2 \leq C\,\alpha_l\,\varepsilon_\pi$. By subadditivity of the spectral norm, $\|A \otimes B\|_2 = \|A\|_2\,\|B\|_2$ and $\|w_l w_l^\top\|_2 = \|w_l\|_2^2$ \citep{horn_matrix_2012}.  Thus,
  \begin{equation}
    \eta \;=\; \|K - K_\pi\|_2
    \;\leq\; \sum_{l=1}^{L} \|w_l w_l^\top\|_2 \, \|K_l - K_{\pi,l}\|_2
    \;\leq\; C\,\varepsilon_\pi \sum_{l=1}^{L} \|w_l\|_2^2\,\alpha_l \, .
  \end{equation}
  Since $K, K_\pi \succeq 0$, applying \cref{lem:posterior_perturbation} with this $\eta$ yields the stated bounds on $\|\mu - \mu_\pi\|_2$ and $\|\Sigma - \Sigma_\pi\|_2$.
\end{proof}

\section{Proof of Proposition 1} \label{app:prop1}
\begin{proof}
  Let $T(\pi) := \max_i s_{\pi(i)} = \sum_{i=2}^{C} \Delta_i$ denote the chain coordinate range from starting point $\pi(1)$. Since $s_{\pi(1)} = 0$ and the chain visits each point of $X$ exactly once, $T(\pi)$ equals the total length of the greedy nearest-neighbor Hamiltonian path on $X$ initiated at $\pi(1)$. Let $T^\star$ denote the length of the shortest Hamiltonian path on $X$ under the Euclidean metric.

  \citet{rosenkrantz_analysis_1977} prove that, for any metric instance on $C$ points, the greedy nearest-neighbor heuristic produces a Hamiltonian path of length at most
  \begin{equation}
    T(\pi) \;\le\; \frac{1}{2}\bigl(\lceil\log_2 C\rceil + 1\bigr)\,T^\star
  \end{equation}
  regardless of the starting point. Applying this upper bound to $\pi_1$ and the trivial lower bound $T^\star \le T(\pi_2)$ to $\pi_2$ gives the claimed inequality
  \begin{equation}
    T(\pi_1) \;\le\; \frac{1}{2}\bigl(\lceil\log_2 C\rceil + 1\bigr)\,T(\pi_2) \, .
  \end{equation}

  For any $i,j$, $\|x_i - x_j\| \le \mathrm{diam}(X) \le T(\pi)$, where the last inequality holds because any Hamiltonian path on $X$ is at least as long as the Euclidean diameter. Furthermore, $|s_{\pi(i)} - s_{\pi(j)}| \le T(\pi)$ and thus, $\varepsilon_\pi \le T(\pi)$. Hence the upper bound of \cref{thm:lmc_posterior_perturbation} inherits the same $\mathcal{O}(\log C)$ variation across choices of starting point.
\end{proof}

\end{document}